\documentclass{article}

\usepackage{tikz}
\usepackage{pifont}

\usepackage[utf8]{inputenc} 
\usepackage[T1]{fontenc}    
\usepackage{hyperref}       
\usepackage{url}            
\usepackage{booktabs}       
\usepackage{amsfonts}       
\usepackage{nicefrac}       
\usepackage{microtype}      
\usepackage{xcolor}         
\usepackage{float} 
\usepackage{graphicx} 
\usepackage[caption=false,font=normalsize,labelfont=sf,textfont=sf]{subfig}

\usepackage{bm} 

\usepackage[margin=1.03in]{geometry} 

\usepackage{algorithm} 
\usepackage{algorithmic}

\usepackage{wrapfig}

\usepackage{enumerate} 

\usepackage{amsmath}
\usepackage{amssymb}
\usepackage{mathtools}
\usepackage{amsthm}
\usepackage{mathrsfs}

\newtheorem{theorem}{Theorem}
 
\newtheorem{definition}{Definition}
\newtheorem{lemma}{Lemma}
\newtheorem{remark}{Remark}
\newtheorem{corollary}{Corollary}

\newtheorem*{theorem*}{Theorem}
\newtheorem*{example*}{Example} 
\newtheorem*{definition*}{Definition}
\newtheorem*{lemma*}{Lemma}
\newtheorem*{remark*}{Remark}
\newtheorem*{corollary*}{Corollary}
\newtheorem*{proposition*}{Proposition}
\newtheorem*{assumption*}{Assumption}
\newtheorem*{claim*}{Claim}

\newtheoremstyle{TheoremNum}
        {\topsep}{\topsep}              
        {\itshape}                      
        {}                              
        {\bfseries}                     
        {.}                             
        { }                             
        {\thmname{#1}\thmnote{ \bfseries #3}}
\theoremstyle{TheoremNum}

\newtheoremstyle{LemmaNum}
        {\topsep}{\topsep}              
        {\itshape}                      
        {}                              
        {\bfseries}                     
        {.}                             
        { }                             
        {\thmname{#1}\thmnote{ \bfseries #3}}
\theoremstyle{LemmaNum}

\renewcommand{\Pr}{ \mathbb{P} }

\newcommand{\I}{ \mathcal{I} }

\newcommand{\A}{ \mathcal{A} }

\newcommand{\wtr}{\widetilde{r}}
\newcommand{\B}{\mathbb{B}}

\newcommand{\E}{\mathbb{E}}
\newcommand{\N}{\mathcal{N}}

\renewcommand{\stop}{\mathrm{stop}}

\renewcommand{\[}{\left[ }
\renewcommand{\]}{\right] }

\renewcommand{\(}{\left( }
\renewcommand{\)}{\right) }

\newcommand{\wt}{\widetilde }
\newcommand{\wh}{\widehat }

\newcommand{\x}{\mathbf{x}}
\newcommand{\y}{\mathbf{y}}

\title{Lipschitz Bandits with Batched Feedback}

\author{Yasong Feng\thanks{Shanghai Center for Mathematical Sciences, Fudan University; email: \texttt{ysfeng20@fudan.edu.cn}.}\and
        Zengfeng Huang\thanks{School of Data Science, Fudan University; email: \texttt{huangzf@fudan.edu.cn}.}\and
	Tianyu Wang\thanks{Shanghai Center for Mathematical Sciences, Fudan University; email: \texttt{wangtianyu@fudan.edu.cn}.
}}
\date{}
\begin{document}

\maketitle

\begin{abstract}
    \noindent In this paper, we study Lipschitz bandit problems with batched feedback, where the expected reward is Lipschitz and the reward observations are communicated to the player in batches. We introduce a novel landscape-aware algorithm, called Batched Lipschitz Narrowing (BLiN), that optimally solves this problem. Specifically, we show that for a $T$-step problem with Lipschitz reward of zooming dimension $d_z$, our algorithm achieves theoretically optimal (up to logarithmic factors) regret rate $\widetilde{\mathcal{O}}\left(T^{\frac{d_z+1}{d_z+2}}\right)$ using only $ \mathcal{O} \left( \log\log T\right) $ batches. We also provide complexity analysis for this problem. Our theoretical lower bound implies that $\Omega(\log\log T)$ batches are necessary for any algorithm to achieve the optimal regret. Thus, BLiN achieves optimal regret rate (up to logarithmic factors) using minimal communication.
\end{abstract}

\setlength{\parindent}{0pt}

\section{Introduction}\label{intro} 
    Multi-Armed Bandit (MAB) algorithms aim to exploit the good options while explore the decision space. These algorithms and methodologies find successful applications in artificial intelligence and reinforcement learning (e.g., \cite{silver2016mastering}). 
	While the classic MAB setting assumes that the rewards are immediately observed after each arm pull, real-world data often arrives in different patterns. 
	For example, observations from clinical trials are often be collected in a batched fashion \cite{pocock1977group}. Another example is from online advertising, where strategies are tested on multiple customers at the same time \cite{bertsimas2007learning}. In such cases, any observation-dependent decision-making should comply with this data-arriving pattern, including MAB algorithms. 
	
	
	In this paper, we study the Lipschitz bandit problem with batched feedback -- a MAB problem where the expected reward is Lipschitz and the reward observations are communicated to the player in batches. In such settings, rewards are communicated only at the end of the batches, and the algorithm can only make decisions based on information up to the previous batch. Existing Lipschitz bandit algorithms heavily rely on timely access to the reward samples, since the partition of arm space may change at any time. Therefore, they can not solve the batched feedback setting. To address this difficulty, we present a novel adaptive algorithm for Lipschitz bandits with communication constraints, named \textit{Batched Lipschitz Narrowing} (BLiN). BLiN learns the landscape of the reward by adaptively narrowing the arm set, so that regions of high reward are more frequently played. 
	Also, BLiN determines the data collection procedure adaptively, so that only very few data communications are needed. 
	
	The above BLiN procedure achieves optimal regret rate $\wt{\mathcal{O}} \left( T^{\frac{d_z + 1}{d_z + 2}} \right)$ ($d_z$ is the zooming dimension \cite{kleinberg2008multi,bubeck2011x}), and can be implemented in a clean and friendly form. 
	In addition to achieving the optimal regret rate, BLiN also improves the state-of-the-art results in the following senses: 
    \leftmargini=4mm
    \begin{itemize}
	    \item BLiN's communication complexity is optimal. BLiN only needs $  \mathcal{O}(\log\log T) $ rounds of communications to achieve the optimal regret rate (Theorem \ref{thm:ablin}), and no algorithm can achieve this rate with fewer than $  \Omega(\log\log T) $ rounds of communications (Corollary \ref{coro}).
	    \item BLiN's time complexity is optimal: if the arithmetic operations and sampling are of complexity $\mathcal{O} (1)$, then the time complexity of BLiN is $\mathcal{O} (T)$, which improves the best known time complexity $\mathcal{O} (T \log T)$ for Lipschitz bandit problems \cite{bubeck2011x}.
            \item The space complexity of BLiN is $\mathcal{O}\left(T^{\frac{d_z+1}{d_z+2}} (\log T)^{-\frac{d_z+1}{d_z+2}}\right)$, which also improves the best known result. This is because we do not need to store information of cubes in previous batches. The detailed time and space complexity analysis of BLiN is in Remark \ref{rk:complexity}.
	\end{itemize} 

 In Table \ref{tab:lip} we provide a comparison of BLiN and state-of-the-art Lipschitz bandit algorithms in terms of regret bound, communication bound, time complexity and space complexity.

     \begin{table*}[th]
        \centering
        \caption{Comparison with State-of-the-art Lipschitz Bandit Algorithms}
        \begin{tabular}{ccccc}
             \toprule
             algorithm & regret & communication & time complexity & space complexity\\
             \midrule
             Zooming \cite{kleinberg2008multi} & $\widetilde{\mathcal{O}}\left(T^{\frac{d_z+1}{d_z+2}}\right)$ & $T$ & $\mathcal{O}\left(T^2\right)$ & $\mathcal{O}(T)$ \\
             HOO \cite{bubeck2011x} & $\widetilde{\mathcal{O}}\left(T^{\frac{d_z+1}{d_z+2}}\right)$ & $T$ & $\mathcal{O}\left(T\log T\right)$ & $\mathcal{O}(T)$ \\
             A-BLiN (our work) & $\bm{\widetilde{\mathcal{O}}\left(T^{\frac{d_z+1}{d_z+2}}\right)}$ & $\bm{\mathcal{O}(\log\log T)}$ & $\bm{\mathcal{O}\left(T\right)}$ & $\bm{\mathcal{O}\left(T^{\frac{d_z+1}{d_z+2}} (\log T)^{-\frac{d_z+1}{d_z+2}}\right)}$ \\
             \bottomrule
        \end{tabular}
        \label{tab:lip}
    \end{table*}
	
	
\subsection{Settings \& Preliminaries} \label{prel}

For a Lipschitz bandit problem (with communication constraints), the arm set is a compact doubling metric space $( \mathcal{A}, d_\mathcal{A} )$. The expected reward $\mu : \mathcal{A} \rightarrow \mathbb{R}$ is $1$-Lipschitz with respect to the metric $d_{\mathcal{A}}$, that is, 
$ |\mu(x_1)-\mu(x_2)|\leq d_\mathcal{A}(x_1,x_2)$ for any $ x_1,x_2\in\mathcal{A} $.


At time $t\leq T$, the learning agent  pulls an arm $x_t\in\mathcal{A}$ that yields a reward sample $y_t = \mu (x_t) + \epsilon_t $, where $ \epsilon_{t} $ is a mean-zero independent sub-Gaussian noise. Without loss of generality, we assume that $\epsilon_t\sim\mathcal{N}(0,1)$, since generalizations to other sub-Gaussian noises are not hard. 




Similar to most bandit learning problems, the agent seeks to minimize regret in the batched feedback environment. The regret is defined as $R(T)=\sum_{t=1}^T\left(\mu^*-\mu(x_t)\right)$, where $\mu^*$ denotes $\max_{x\in\mathcal{A}}\mu(x)$. For simplicity, we define $\Delta_x=\mu^*-\mu(x)$ (called optimality gap of $x$) for all $x\in\mathcal{A}$. 
\subsubsection{Doubling Metric Spaces and the $([0,1]^d, \| \cdot \|_\infty)$ Metric Space}

By the Assouad's embedding theorem \cite{assouad1983plongements}, the (compact) doubling metric space $ ( \mathcal{A} , d_{\mathcal{A}} ) $ can be embedded into a Euclidean space with some distortion of the metric; See \cite{pmlr-v119-wang20q} for more discussions in a machine learning context. Due to existence of such embedding, the metric space $ ([0,1]^d, \| \cdot \|_{\infty})  $, where metric balls are hypercubes, is sufficient for the purpose of our paper. For the rest of the paper, we will use hypercubes in algorithm design for simplicity, while our algorithmic idea generalizes to other doubling metric spaces.

\subsubsection{Zooming Number and Zooming Dimension} 

An important concept for bandit problems in metric spaces is the zooming number and the zooming dimension \cite{kleinberg2008multi, bubeck2008tree, slivkins2011contextual}, which we discuss now. We start with the definition of packing numbers.
\begin{definition}
    Let $(\A,d_\A)$ be a metric space. The $r$-packing number $\N(\mathcal{S},r)$ of $\mathcal{S}\subset\A$ is the size of the largest packing of $\mathcal{S}$ with disjoint $d_\A$-open balls with radius $r$.
\end{definition}
Then we define the zooming number and the zooming dimension.
\begin{definition}
    For a problem instance with arm set $\mathcal{A}$ and expected reward $\mu$, we let $S(r)$ denote the set of $r$-optimal arms, that is, $ S (r) = \{ x \in \A : \Delta_x=\mu^*-\mu(x) \le r \} $. We define the $r$-zooming number $N_r$ as $N_r=\N\left(S(16r),\frac{r}{2}\right)$. The zooming dimension is then defined as 
    \begin{align*} 
        d_z :=\; \min \left\{d\geq0: \exists a>0,\;N_r \le ar^{-d},\;\forall 0<r<1 \right\}. 
    \end{align*}
    Moreover, we define the zooming constant $C_z$ as
    \begin{align*}
        C_z=\min\left\{a>0:\;N_r \le ar^{-d_z},\;\forall 0<r<1 \right\}.
    \end{align*}
\end{definition}

Zooming dimension $d_z$ can be significantly smaller than ambient dimension $d$ and can be zero. For a simple example, consider a problem with ambient dimension $d=1$ and expected reward function $\mu(x)=x$ for $0\leq x\leq1$. Then for any $r=2^{-i}$ with $i\geq4$, we have $S(16r)=[1-16r,1]$ and $N_r=16$. Therefore, for this problem the zooming dimension equals to $0$, with zooming constant $C_z=16$.

\subsection{Batched feedback pattern and our results}
In the batched feedback setting, for a $T$-step game, the player determines a grid $\mathcal{T}=\{t_0,\cdots,t_B\}$ \emph{adaptively}, where $0=t_0<t_1<\cdots<t_B=T$ and $B\ll T$. During the game, reward observations are communicated to the player only at the grid points $t_1, \cdots, t_B$. As a consequence, for any time $t$ in the $j$-th batch, that is, $t_{j-1}<t\leq t_j$, the reward $y_t$ cannot be observed until time $t_j$, and the decision made at time $t$ depends only on rewards up to time $t_{j-1}$. The determination of the grid $\mathcal{T}$ is adaptive in the sense that the player chooses each grid point $t_j\in\mathcal{T}$ based on the operations and observations up to the previous point $t_{j-1}$.

In this work, we present BLiN algorithm to solve Lipschitz bandits under batched feedback. During the learning procedure, BLiN detects and eliminates the `bad area' of the arm set in batches and partition the remaining area according to an approporiate \emph{edge-length sequence}. Our first theoretical upper bound is that with simple Doubling Edge-length Sequence $r_m=2^{-m+1}$, BLiN achieves optimal regret rate $\widetilde{\mathcal{O}}\left(T^{\frac{d_z+1}{d_z+2}}\right)$ by using $\mathcal{O}(\log T)$ batches.
\begin{theorem}\label{thm:dblin}
    With probability exceeding $1-\frac{2}{T^6}$, the $T$-step total regret $R(T)$ of BLiN with Doubling Edge-length Sequence (D-BLiN) satisfies
    \begin{equation*}
        R(T)\lesssim T^\frac{d_z+1}{d_z+2}\cdot (\log T)^{\frac{1}{d_z+2}},
    \end{equation*}
    where $d_z$ is the zooming dimension of the problem instance. In addition, D-BLiN only needs no more than $\mathcal{O}(\log T)$ rounds of communications to achieve this regret rate. Here and henceforth, $\lesssim$ only omits constants. 
\end{theorem} 
While D-BLiN is efficient for batched Lipschitz bandits, its communication complexity is not optimal. We then propose a new edge-length sequence, which we call Appropriate Combined Edge-length Sequence (ACE Sequence) to improve the algorithm. The idea behind this sequence is that by appropriately combining some batches, the algorithm can achieve better communication bound without incurring increased regret. As we shall see, BLiN with ACE Sequence (A-BLiN) achieves regret rate $\widetilde{\mathcal{O}}\left(T^{\frac{d_z+1}{d_z+2}}\right)$ with only $\mathcal{O}(\log\log T)$ batches.
\begin{theorem}\label{thm:ablin}
    With probability exceeding $1-\frac{2}{T^6}$, the $T$-step total regret $R(T)$ of A-BLiN satisfies
    \begin{equation*}
        R(T)\lesssim T^\frac{d_z+1}{d_z+2}\cdot(\log T)^{\frac{1}{d_z+2}}\cdot\log\log T,
    \end{equation*}
    where $d_z$ is the zooming dimension of the problem instance. In addition, Algorithm \ref{alg_Op} only needs no more than $\mathcal{O}(\log\log T)$ rounds of communications to achieve this regret rate. 
\end{theorem}

As a comparison, seminal works \cite{kleinberg2008multi,slivkins2011contextual,bubeck2011x} show that the optimal regret bound for Lipschitz bandits without communications constraints, where the reward observations are immediately observable after each arm pull, is $R(T)\lesssim T^{\frac{d_z+1}{d_z+2}}\cdot(\log T)^{\frac{1}{d_z+2}}$.
Therefore, A-BLiN achieves optimal regret rate of Lipschitz bandits by using very few batches.

Furthermore, we provide a theoretical lower bound for Lipschitz bandits with batched feedback.

\begin{theorem}
    \label{thm:lower-adaptive-dz-intro}
    Consider Lipschitz bandit problems with time horizon $T$, ambient dimension $d$ and zooming dimension $d_z\leq d$. If $B$ rounds of communications are allowed, then for any policy $\pi$, there exists a problem instance with zooming dimension $d_z$ such that 
	\begin{align*} 
	    \E \left[ R_T(\pi) \right] 
	    \geq \frac{1}{512B^2}T^{\frac{1-\frac{1}{d_z+2}}{1-\left(\frac{1}{d_z+2}\right)^B}}.
	\end{align*}
\end{theorem}

In the lower bound analysis, we use a ``linear-decaying extension'' technique to construct problem instances with zooming dimension $d_z$. To the best of our knowledge, our construction provides the first minimax lower bound for Lipschitz bandits where the zooming dimension $d_z$ is explicitly different from the ambient dimension $d$. As a result of Theorem \ref{thm:lower-adaptive-dz-intro}, we can derive the minimum rounds
of communications needed to achieve optimal regret bound for
Lipschitz bandit problem, which is stated in Corollary \ref{coro}. The proof of Corollary \ref{coro} is deferred to Appendix \ref{app:coro}.

\begin{corollary}\label{coro}
    For Lipschitz bandit problems with ambient dimension $d$, zooming dimension $d_z\leq d$ and time horizon $T$, any algorithm needs $\Omega(\log\log T)$ rounds of communications to achieve the optimal regret rate $\mathcal{O}\left(T^{\frac{d_z+1}{d_z+2}}\right)$. 
\end{corollary} 

Consequently, BLiN algorithm is optimal in terms of both regret and communication.
	
\subsection{Related Works} 
    The history of the Multi-Armed Bandit (MAB) problem can date back to Thompson \cite{thompson1933likelihood}. Solvers for this problems include the UCB algorithms \cite{lai1985asymptotically, agrawal1995sample, auer2002finite}, the arm elimination method \cite{even2006action, perchet2013multi, salgia2021domain}, the $\epsilon$-greedy strategy \cite{auer2002finite,sutton2018reinforcement}, the exponential weights and mirror descent framework \cite{auer2002nonstochastic}. 
    
    Recently, with the prevalence of distributed computing and large-scale field experiments, the setting of batched feedback has captured attention (e.g., \cite{cesa2013online}). Perchet et al. \cite{perchet2016batched} mainly consider batched bandit with two arms, and a matching lower bound for static grid is proved. 
    It was then generalized by Gao et al. \cite{gao2019batched} to finite-armed bandit problems. In their work, the authors designed an elimination method for finite-armed bandit problem and proved matching lower bounds for both static and adaptive grid. Soon afterwards, Zhang et al. \cite{zhang2020inference} studies inference for batched bandits. Esfandiari et al. \cite{esfandiari2021regret} studies batched linear bandits and batched adversarial bandits. Han et al. \cite{han2020sequential} and Ruan et al. \cite{ruan2021linear} provide solutions for batched contextual linear bandits. Li and Scarlett \cite{li2022gaussian} studies batched Gaussian process bandits. Batched dueling bandits have also been studied by Agarwal et al. \cite{agarwal2022batched}. Parallel to the regret control regime, best arm identification with limited number of batches was studied in \cite{agarwal2017learning} and \cite{jun2016top}. Top-$k$ arm identification in the collaborative learning framework is also closely related to the batched setting, where the goal is to minimize the number of iterations (or communication steps) between agents. In this setting, tight bounds have been obtained in the recent works \cite{tao2019collaborative,karpov2020collaborative}.
    Yet the problem of Lipschitz bandit with communication constraints remains unsolved.
    
    The Lipschitz bandit problem is important in its own stand. 
    The Lipschitz bandit problem was introduced as ``continuum-armed bandits'' \cite{agrawal1995continuum}, where the arm space is a compact interval. Along this line, bandits that are Lipschitz (or H\"older) continuous have been studied. For this problem, Kleinberg \cite{kleinberg2005nearly} proves a $\Omega (T^{2/3})$ lower bound and introduced a matching algorithm. Under extra conditions on top of Lipschitzness, regret rate of $\widetilde{\mathcal{O}} (T^{1/2})$ was achieved \cite{auer2007improved, cope2009regret}. For general (doubling) metric spaces, the Zooming bandit algorithm \cite{kleinberg2008multi} and the Hierarchical Optimistic Optimization (HOO) algorithm \cite{bubeck2011x} were developed. In more recent years, some attention has been focused on Lipschitz bandit problems with certain extra structures. To name a few, Bubeck et al. \cite{bubeck2011lipschitz} study Lipschitz bandits for differentiable rewards, which enables algorithms to run without explicitly knowing the Lipschitz constants. Wang et al. \cite{10.1145/3412815.3416885} studied discretization-based Lipschitz bandit algorithms from a Gaussian process perspective. Magureanu et al. \cite{magureanu2014lipschitz} derive a new concentration inequality and study discrete Lipschitz bandits. The idea of robust mean estimators \cite{bickel1965some, alon1999space, bubeck2013bandits} was applied to the Lipschitz bandit problem to cope with heavy-tail rewards, leading to the development of a near-optimal algorithm for Lipschitz bandit with heavy-tailed rewards \cite{lu2019optimal}. Wanigasekara and Yu \cite{christina2019nonparametric} studied Lipschitz bandits where a clustering is used to infer the underlying metric. Contextual Lipschitz bandits have been studied by Slivkins \cite{slivkins2011contextual}. Contextual bandits with continuous actions have also been studied by Krishnamurthy et al. \cite{krishnamurthy2020contextual} and Majzoubi et al. \cite{majzoubi2020efficient} through a smoothness approach. Yet all of the existing works for Lipschitz bandits assume that the reward sample is immediately observed after each arm pull, and none of them solve the Lipschitz bandit problem with communication constraints.

This paper is organized as follows. In section \ref{s_alg}, we introduce the BLiN algorithm and give a visual illustration of the algorithm procedure. In section \ref{s_reg}, we prove that BLiN with ACE Sequence achieves the optimal regret rate using only $\mathcal{O}\left(\log\log T\right)$ rounds of communications. Section \ref{s_lb} provides information-theoretical lower bounds for Lipschitz bandits with communication constraints, which shows that BLiN is optimal in terms of both regret and rounds of communications. Experimental results are presented in Section \ref{exp}.

\section{Algorithm} \label{s_alg} 

With communication constraints, the agent's knowledge about the environment does not accumulate within each batch.
This characteristic of the problem suggests a `uniform' type algorithm -- we shall treat each step within the same batch equally. Following this intuition, in each batch, we uniformly play the remaining arms, and then eliminate arms of low reward after the observations are communicated. Next we describe the uniform play rule and the arm elimination rule.

\textbf{Uniform Play Rule:}
At the beginning of each batch $m$, a collection of subsets of the arm space $ \A_m  = \{ C_{m,1}, C_{m,2}, \cdots, C_{m, |\A_m|} \} $ is constructed. This collection of subset $\A_m$ consists of standard cubes, and all cubes in $\A_m$ have the same edge length $r_m$. We will detail the construction of $\A_m$ when we describe the arm elimination rule. We refer to cubes in $\mathcal{A}_m$ as active cubes of batch $m$.

During batch $m$, each cube in $\A_m$ is played
\begin{equation}\label{eq:def-nm}
    n_m \triangleq \frac{16 \log T}{r_m^2}
\end{equation}
times, where $T$ is the total time horizon. More specifically, within each $ C \in \A_m $, arms $x_{C,1}, x_{C,2}, \cdots, x_{C, n_m} \in C$ are played.\footnote{One can arbitrarily play $ x_{C,1}, x_{C,2}, \cdots, x_{C, n_m} $ as long as $ x_{C,i} \in C $ for all $i$.} The reward samples $ \left\{ y_{C,1}, y_{C,2}, \cdots, y_{C, n_m} \right\}_{C \in \A_m} $ corresponding to $ \left\{ x_{C,1}, x_{C,2}, \cdots, x_{C, n_m} \right\}_{C \in \A_m} $ will be collected at the end of the this batch. 

\textbf{Arm Elimination Rule:} 
At the end of batch $ m $, information from the arm pulls is collected, and we estimate the reward of each $C \in \A_m$ by $\wh{\mu}_{ m } (C) = \frac{ 1 }{n_m} \sum_{i=1}^{n_m} y_{C,i} $. Cubes of low estimated rewards are then eliminated, according to the following rule: a cube $C \in \A_m$ is eliminated if $\wh{\mu}_{m}^{\max} - \wh{\mu}_{m} (C) \ge 4 r_m $, where $ \wh{\mu}_{m}^{\max} := \max_{ C \in \A_m } \wh{\mu}_{m} (C) $.
After necessary removal of ``bad cubes'', each cube in $\A_m$ that survives the elimination is equally partitioned into $\left(\frac{r_m}{r_{m+1}}\right)^d$ subcubes of edge length $ r_{m+1} $, where $\{r_{m}\}_m$ is predetermined sequence to be specified soon. These cubes (of edge length $ r_{m+1} $) are collected to construct $\A_{m+1}$, and the learning process moves on to the next batch. Appropriate rounding may be required to ensure the ratio $\frac{r_m}{r_{m+1}}$ is an integer. See Remark \ref{remark:race} for more details. 

The learning process is summarized in Algorithm \ref{alg_Op}. 




\begin{algorithm}[ht]
	\caption{Batched Lipschitz Narrowing (BLiN)} 
	\label{alg_Op} 
	\begin{algorithmic}[1]  
		\STATE \textbf{Input.} Arm set $\A=[0,1]^d$; time horizon $T$.
		\STATE \textbf{Initialization} Number of batches $B$; Edge-length sequence $\{r_m\}_{m=1}^{B+1}$; The first grid point $t_0=0$; Equally partition $\mathcal{A}$ to $r_1^d$ subcubes and define $\mathcal{A}_{1}$ as the collection of these subcubes. 
		\STATE Compute $n_m=\frac{16\log T}{r_m^2}$ for $m=1,\cdots,B+1$.
		
		\FOR{$m=1,2,\cdots,B$}
			\STATE For each cube $C\in\mathcal{A}_m$, play arms $x_{C,1},\cdots x_{C,n_m}$ from $C$.
			\STATE Collect the rewards of all pulls up to $t_m$. Compute the average payoff $\wh{\mu}_m(C)=\frac{\sum_{i=1}^{n_m}y_{C,i}}{n_m}$ for each cube $C\in\mathcal{A}_m$.
			Find  $\wh{\mu}_m^{max}=\max_{C\in\mathcal{A}_m}\wh{\mu}(C)$.
			\STATE For each cube $C\in\mathcal{A}_m$, eliminate $C$ if $\wh{\mu}_m^{max}-\wh{\mu}_m(C)>4r_m$. Let $\mathcal{A}_m^+$ be set of cubes not eliminated.
			\STATE Compute $t_{m+1}=t_m+(r_{m}/r_{m+1})^d\cdot|\mathcal{A}_m^+|\cdot n_{m+1}$. If $t_{m+1}\geq T$ or $m=B$ then \textbf{break}.
			\STATE Equally partition each cube in $\mathcal{A}_m^+$ into $\left(r_m/r_{m+1}\right)^d$ subcubes and define $\mathcal{A}_{m+1}$ as the collection of these subcubes. /*\textit{See Remark \ref{remark:race} for more details on cases where $\left(r_m/r_{m+1}\right)^d$ is not an integer.}*/
		\ENDFOR
    \STATE \textbf{Cleanup:} Arbitrarily play the remaining arms until all $T$ steps are used.  
	\end{algorithmic} 
\end{algorithm} 
\begin{remark}[Time and space complexity] 
    \label{rk:complexity}
    The time complexity of our algorithm is $\mathcal{O}(T)$, which is better than the state of the art $\mathcal{O}(T\log T)$ in \cite{bubeck2011x}. This is because that the running time of a batch $j$ is of order $\mathcal{O}(l_j)$, where $l_j=t_j-t_{j-1}$ is number of samples in batch $j$. Since $\sum_j l_j=T$, the time complexity of BLiN is $\mathcal{O}(T)$. Besides, the space complexity of our algorithm is $\mathcal{O}\left(T^{\frac{d_z+1}{d_z+2}} (\log T)^{-\frac{d_z+1}{d_z+2}}\right)$, which also improves the best known space complexity. This is because we do not need to store information of cubes in previous batches. The space complexity analysis is deferred to Appendix \ref{app:space}.
\end{remark}

The following theorem gives regret and communication upper bound of BLiN with Doubling Edge-length Sequence $r_m=2^{-m+1}$ (see Appendix \ref{proof:dblin} for a proof). 
Note that this result implies Theorem \ref{thm:dblin}.

\begin{theorem}\label{main_t_d}
    With probability exceeding $1-\frac{2}{T^6}$, the $T$-step total regret $R(T)$ of BLiN with Doubling Edge-length Sequence (D-BLiN) satisfies
    \begin{equation*}
        R(T)\leq(512C_z+16)\cdot T^\frac{d_z+1}{d_z+2} (\log T)^{\frac{1}{d_z+2}},
    \end{equation*}
    where $d_z$ is the zooming dimension of the problem instance. In addition, D-BLiN only needs no more than $\frac{\log T-\log\log T}{d_z + 2} + 2$ rounds of communications to achieve this regret rate.
\end{theorem}

Although D-BLiN efficiently solves batched Lipschitz bandits, its simple partition strategy leads to suboptimal communication complexity. Now we show that by approporiately combining some batches, BLiN achieves the optimal communication bound, without incurring additional regret. Specifically, we introduce the following edge-length sequence, which we call ACE Sequence. When using ACE Sequence, the regret of each batch is of order $\widetilde{\mathcal{O}}\left(T^{\frac{d_z+1}{d_z+2}}\right)$. Thus, the length of any batch cannot be increased without affecting the optimal regret rate. This implies that the ACE Sequence is optimal in terms of both regret and communication. See the proof of Theorem \ref{main_t_opt} for more details.




\begin{definition}
    \label{def:ace}
    For a problem with ambient dimension $d$, zooming dimension $d_z$ and time horizon $T$, we denote $c_1=\frac{d_z+1}{(d+2)(d_z+2)} \log \frac{T}{\log T}$ and $c_{i+1} = {\eta c_i }$ for any $i\ge 1$, where $\eta=\frac{d+1-d_z}{d+2}$. Then the Appropriate Combined Edge-length (ACE) Sequence $\{r_m\}$ is defined by $r_m=2^{-\sum_{i=1}^mc_i}$ for any $m\geq1$.
\end{definition}
    




Theorem \ref{main_t_opt} states that BLiN with ACE Sequence (A-BLiN) obtains an improved communication complexity, thus proves Theorem \ref{thm:ablin}.

\begin{theorem}\label{main_t_opt}
    With probability exceeding $1-\frac{2}{T^6}$, the $T$-step total regret $R(T)$ of Algorithm \ref{alg_Op} satisfies
    \begin{equation}\label{up_bound}
        R(T)\leq\left(\frac{128C_z}{\log\frac{d+2}{d+1-d_z}}\cdot\log\log T+16\right)\cdot T^\frac{d_z+1}{d_z+2}(\log T)^{\frac{1}{d_z+2}},
    \end{equation}
    where $d_z$ is the zooming dimension of the problem instance. In addition, Algorithm \ref{alg_Op} with ACE sequence only needs no more than $\frac{\log\log T}{\log\frac{d+2}{d+1-d_z}}+1$ rounds of communications to achieve this regret rate. 
\end{theorem}


The partition and elimination process of a real A-BLiN run is in Figure \ref{mlos}. 
In the $i$-th subgraph, the white cubes are those remaining after the $(i-1)$-th batch. In this experiment, we set $\mathcal{A}=[0,1]^2$, and the optimal arm is $x^*=(0.8,0.7)$. Note that $x^*$ is not eliminated during the game. More details of this experiment are in Section \ref{exp}.
\begin{figure*}[ht]
    \centering
    \subfloat{ 
    	\includegraphics[width=4.7cm]{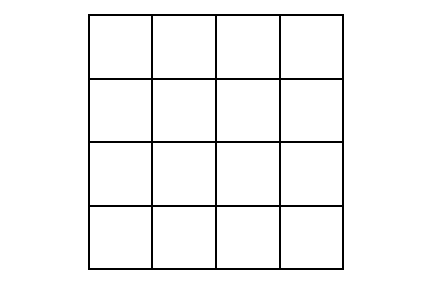}
    }
    \hspace{-15mm}
	\subfloat{
		\includegraphics[width=4.7cm]{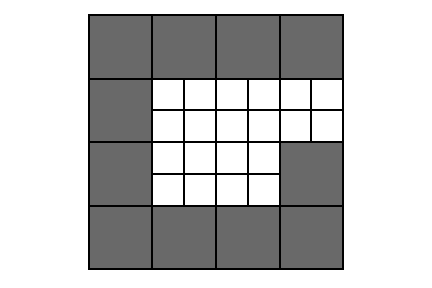}
	}
    \hspace{-15mm}
	\subfloat{
		\includegraphics[width=4.7cm]{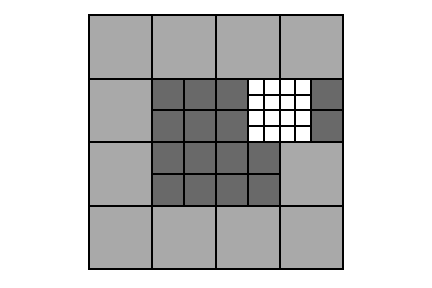}
	}
    \hspace{-15mm}
	\subfloat{
		\includegraphics[width=4.7cm]{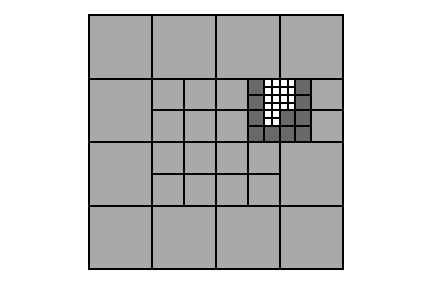}
	}
	\caption{Partition and elimination process of A-BLiN. The $i$-th subfigure shows the pattern before the $i$-th batch, which is from a real A-BLiN run on the reward function defined in Section \ref{exp}. 
	Dark gray cubes are those eliminated in the most recent batch, while the light gray ones are those eliminated in earlier batches. 
	For the total time horizon $T= 80000$, A-BLiN needs $4$ rounds of communications. For this experiment, $r_1=\frac{1}{4},\;r_2=\frac{1}{8},\;r_3=\frac{1}{16},\;r_4=\frac{1}{32},\;$ which is the ACE sequence (rounded as in Remark \ref{remark:race}) for $ d = 2 $ and $d_z = 0$.} 
	\label{mlos} 
\end{figure*}

The definition of the ACE Sequence relies on the zooming dimension $d_z$. If $d_z$ is not known ahead of time, we recommend two ways to proceed. 1) The player can apply BLiN with Doubling Edge-length Sequence $r_m=2^{-m+1}$ (D-BLiN). Theorem \ref{main_t_d} shows that D-BLiN achieves optimal regret rate $\widetilde{\mathcal{O}}\left(T^{\frac{d_z+1}{d_z+2}}\right)$ by using $\mathcal{O}(\log T)$ batches. Although the rate $\mathcal{O}(\log T)$ is not optimal, it is good enough for total time horizon $T$ that is not too large, as shown in the experimental results in Appendix \ref{app:exp-dblin}. 2) If an upper bound $d_u$ of the zooming dimension is known, that is, $d_z\leq d_u\leq d$, then the player can apply A-BLiN by using $d_u$ to define the ACE Sequence. Theorem \ref{main_t_opt} yields that A-BLiN with $d_u$ achieves regret rate $\widetilde{\mathcal{O}}\left(T^{\frac{d_u+1}{d_u+2}}\right)$ by using $\mathcal{O}(\log\log T)$ batches.

\section{Regret Analysis of A-BLiN} \label{s_reg}


In this section, we provide regret analysis for A-BLiN. The highlight of the finding is that $\mathcal{O}(\log\log T)$ batches are sufficient to achieve optimal regret rate of $ \widetilde{\mathcal{O}} \left( T^{\frac{d_z + 1}{d_z + 2}} \right) $, as summarized in Theorem \ref{main_t_opt}. 


To prove Theorem \ref{main_t_opt}, we first show that the estimator $ \wh{\mu} $ is concentrated to the true expected reward $ \mu $ (Lemma \ref{lem:concen}), and the optimal arm survives all eliminations with high probability (Lemma \ref{nel}).
    In the following analysis, we let $B_{\stop}$ be the total number of batches of the BLiN run.

\begin{lemma} 
    \label{lem:concen} 
    Define 
    \begin{align*} 
        \mathcal{E} :=  \;  
        \Bigg\{|\mu(x)-\wh{\mu}_m(C)| \leq r_m + \sqrt{\frac{ 16 \log T}{n_m}}, \; \forall 1 \le m \le B_{\stop}-1, \; \forall C \in \A_m, \;\forall x\in C\Bigg\} . 
    \end{align*} 
    It holds that $ \Pr \left( \mathcal{E} \right) \ge 1 - 2 T^{-6} $. 
\end{lemma} 

\begin{proof}
    Fix a cube $C\in\mathcal{A}_m$. Recall the average payoff of cube $C \in \mathcal{A}_m$ is defined as 
    \begin{equation*}
    \wh{\mu}_m(C)=\frac{\sum_{i=1}^{n_m}y_{C,i}}{n_m}.
    \end{equation*}
    We also have
    \begin{equation*}
        \E \left[ \wh{\mu}_m(C) \right] =\frac{\sum_{i=1}^{n_m}\mu(x_{C,i})}{n_m}.
    \end{equation*}
    
    Since $\wh{\mu}_m (C)-\E \left[ \wh{\mu}_m(C) \right]$ obeys normal distribution $\mathcal{N}\left(0,\frac{1}{n_m}\right)$, Hoeffding inequality gives 
    \begin{align*}
        \Pr \left(\left|\wh{\mu}_m (C)-\E \left[ \wh{\mu}_m(C) \right]\right| \ge \sqrt{\frac{16 \log T}{n_m}}\right) \le \frac{2}{T^{8}}.
    \end{align*}
    On the other hand, by Lipschitzness of $\mu$, it is obvious that 
    \begin{align*}
        \left|\E \left[\wh{\mu}_m(C)\right]-\mu(x)\right|\leq r_m, \quad \forall x \in C.
    \end{align*} 
    Consequently, we have
    \begin{align*} 
        \Pr\left(\sup_{x\in C}|\mu(x)-\wh{\mu}_m(C)|\leq r_m+\sqrt{\frac{16\log T}{n_m}}\right)\geq1-\frac{2}{T^8}. 
    \end{align*} 
    
    For $1\leq m\leq B_{\stop}-1$, the $m$-th batch is finished, so any cube $C\in\mathcal{A}_m$ is played for not less than $1$ time, and thus $|\mathcal{A}_m|\leq T$. From here, by similar argument to Lemma F.1 in \cite{sinclair2020adaptive} and Lemma 1 in \cite{lu2019optimal}, taking a union bound over $C\in\mathcal{A}_m$ and $1\leq m\leq B_{\stop}-1$ finishes the proof. 
\end{proof}


\begin{lemma}\label{nel} 
    Under event $\mathcal{E}$ (defined in Lemma \ref{lem:concen}), the optimal arm $x^*=\arg\max\mu(x)$ is not eliminated after the first $B_{\stop}-1$ batches.
\end{lemma}

\begin{proof}
    We use $C^*_m$ to denote the cube containing $x^*$ in $\A_m$. Here we proof that $C^*_m$ is not eliminated in round $m$.
    	
    Under event $\mathcal{E}$, for any cube $C \in \A_m$ and $x\in C$, we have 
    \begin{align*} 
        \wh{\mu}(C) - \wh{\mu}(C^*_m)
        \leq
        \mu(x) + \sqrt{\frac{ 16 \log T}{n_m}} + r_m-\mu(x^*) +  \sqrt{\frac{ 16 \log T}{n_m}} + r_m
        \leq 
        4r_m, 
    \end{align*}
    where the second inequality follows from (\ref{eq:def-nm}). Then from the elimination rule, $C^*_m$ is not eliminated. 
\end{proof}

    	
Based on these results, we show the cubes that survive elimination are of high reward. 
\begin{lemma} 
    \label{lem:eli}
    Under event $\mathcal{E}$ (defined in Lemma \ref{lem:concen}), for any $1\leq m\leq B_{\stop}$, any $C\in\mathcal{A}_m$ and any $x\in C$, $\Delta_x$ satisfies 
    \begin{align}
        \Delta_x\leq8r_{m-1}.\label{ineq_eli}
    \end{align}
\end{lemma} 

\begin{proof}
    For $m=1$, (\ref{ineq_eli}) holds directly from the Lipschitzness of $\mu$. For $m>1$, let $C_{m-1}^*$ be the cube in $\A_{m-1}$ such that $x^* \in C_{m-1}^*$. From Lemma \ref{nel}, this cube $C_{m-1}^*$ is well-defined under $\mathcal{E}$. 
    For any cube $C\in\mathcal{A}_m$ and $x\in C$, it is obvious that $x$ is also in the parent of $C$ (the cube in the previous round that contains $C$), which is denoted by $C_{par}$. 
    Thus for any $x \in C$, it holds that
    \begin{align*}
        \Delta_x
        = 
        \mu^*-\mu(x) 
        \leq 
        \wh{\mu}_{m-1}(C^*_{m-1})+\sqrt{\frac{16\log T}{n_{m-1}}} +r_{m-1}-\wh{\mu}_{m-1}(C_{par})+\sqrt{\frac{16\log T}{n_{m-1}}}+r_{m-1}, 
    \end{align*}
    where the inequality uses Lemma \ref{lem:concen}. 
    	
    Equality $\sqrt{\frac{16\log T}{n_{m-1}}}=r_{m-1}$ gives that 
    \begin{align*}
        \Delta_x 
        &\le 
    	\wh{\mu}_{m-1}(C^*_{m-1})-\wh{\mu}_{m-1}(C_{par})+4r_{m-1}. 
    \end{align*}
    It is obvious that $\wh{\mu}_{m-1}(C^*_{m-1})\leq\wh{\mu}_{m-1}^{\max}$. Moreover, since the cube $C_{par}$ is not eliminated, from the elimination rule we have 
    \begin{align*}
        \wh{\mu}_{m-1}^{\max}-\wh{\mu}_{m-1}(C_{par})\le 4r_{m-1}.
    \end{align*}
    Hence, we conclude that $\Delta_x\leq8r_{m-1}$. 
\end{proof}

    
    We are now ready to prove Theorem \ref{main_t_opt}.
    
\begin{proof}[Proof of Theorem \ref{main_t_opt}]
    Let $R_m$ denote regret of the $m$-th batch. Fixing any positive number $B$, the total regret $R(T)$ can be divided into two parts: $R(T)=\sum_{m\leq B}R_m+\sum_{m>B}R_m$. In the following, we bound these two parts separately and then determine $B$ to obtain the upper bound of the total regret. Moreover, we show A-BLiN uses only $\mathcal{O} (\log \log T)$ rounds of communications to achieve the optimal regret. 
    
    Recall that $\mathcal{A}_m$ is set of the active cubes in the $m$-th batch. According to Lemma \ref{lem:eli}, for any $x\in \cup_{C\in\mathcal{A}_m}C$, we have $\Delta_x \le 8 r_{m-1}$. Let $\mathcal{A}_m^+$ be set of cubes not eliminated in batch $m$. Each cube in $\mathcal{A}_{m-1}^+$ is a $\|\cdot\|_\infty$-ball with radius $\frac{r_{m-1}}{2}$, and is a subset of $S(8r_{m-1})$. Therefore, $\mathcal{A}_{m-1}^+$ forms a $\left(\frac{r_{m-1}}{2}\right)$-packing of $S(8r_{m-1})$, and the definition of zooming dimension yields that
    \begin{equation}\label{eq:aplus}
        |\mathcal{A}_{m-1}^+| \le N_{r_{m-1}} \le C_zr_{m-1}^{-d_z}.
    \end{equation} 
    By definition, $r_m = r_{m-1}2^{-c_m}$, so 
    \begin{equation}\label{eq:a_aplus}
        |\mathcal{A}_{m}| = 2^{d {c_m}}|\mathcal{A}_{m-1}^+|.
    \end{equation}
    The total regret of the $m$-th batch is 
    \begin{align}
        R_m 
        = & \; 
        \sum_{C\in\mathcal{A}_m}\sum_{i=1}^{n_m}\Delta_{x_{C,i}}\leq |\mathcal{A}_{m}| \cdot \frac{16\log T}{r_m^2} \cdot 8r_{m-1} \label{eq:factor-rm}\\
        \overset{(i)}{=} &\;  
        2^{d {c_m}}|\mathcal{A}_{m-1}^+|\cdot \frac{16\log T}{r_m^2} \cdot 8r_{m-1}\nonumber\\
        \overset{(ii)}{\leq} & \; 
        2^{d {c_m}}\cdot C_zr_{m-1}^{-d_z+1}\cdot \frac{128\log T}{r_m^2} \nonumber\\
        \overset{(iii)}{=} &\; 
        2^{(\sum_{i=1}^{m-1} c_i) (d_z+1)  + c_m(d+2)}\cdot 128C_z\log T,\nonumber
    \end{align}
    where (i) follows from (\ref{eq:a_aplus}), (ii) follows from (\ref{eq:aplus}), and (iii) follows from the definition of ACE Sequence.

    Define $C_m =(\sum_{i=1}^{m-1} c_i) (d_z+1)  + c_m(d+2)$. Since $c_m = c_{m-1}\cdot \frac{d+1-d_z}{d+2}$, calculation shows that $C_m = (\sum_{i=1}^{m-2} c_i) (d_z+1)  + c_{m-1 } (d+2)+ c_{m-1}(d_z+1-d-2)+ c_m(d+2)=C_{m-1}$. Thus for any $m$, we have $C_m=C_1=c_1(d+2)$. Hence,
    \begin{equation}\label{eq:firstB}
        R_m\leq2^{c_1(d+2)}\cdot128C_z\log T=T^{\frac{d_z+1}{d_z+2}} \cdot 128C_z(\log T)^{\frac{1}{d_z+2}}.
    \end{equation}
    The inequality (\ref{eq:firstB}) holds even if the $m$-th batch does not exist (where we let $R_m=0$) or is not completed. Thus we obtain the first upper bound $\sum_{m\leq B} R_m \leq T^{\frac{d_z+1}{d_z+2}} \cdot 128C_z\cdot B(\log T)^{\frac{1}{d_z+2}}$.
    Lemma \ref{lem:eli} implies that any arm $x$ played after the first $B$ batches satisfies $\Delta_x\leq8r_B$, so the total regret after $B$ batches is bounded by
    \begin{align*}
        \sum_{m>B}R_m
        \leq& \; 
        8r_B\cdot T =8T\cdot2^{-\sum_{i=1}^B c_i}=8T\cdot2^{-c_1 (\frac{1-\eta^B}{1-\eta })}\\
        =&\; 
        8T^{\frac{d_z+ 1}{d_z+2}}(\log T)^{\frac{1}{d_z+2}} \cdot \left(T/\log T\right)^{\frac{\eta^B }{d_z+2}} 
        \leq  \; 
        8T^{\frac{d_z+ 1}{d_z+2}}(\log T)^{\frac{1}{d_z+2}} \cdot T^{\frac{\eta^B }{d_z+2}}.
    \end{align*}
    Therefore, the total regret $R(T)$ satisfies
    \begin{align*}
        R(T)
        = \; 
        \sum_{m\leq B}R_m+\sum_{m>B}R_m 
        \leq \; 
        \left(128C_z\cdot B + 8T^{\frac{\eta^B }{d_z+2}}\right)\cdot T^{\frac{d_z+ 1}{d_z+2}}(\log T)^{\frac{1}{d_z+2}}.
    \end{align*}
    
    This inequality holds for any positive $B$. Then by choosing
    $B^*=\frac{\log\log T-\log(d_z+2)}{\log \frac{1}{\eta}} = \frac{\log\log T-\log(d_z+2)}{\log \frac{d+2}{d+1-d_z}}$, we have $\frac{\eta^{B^*} }{d_z+2} = \frac{1}{\log T}$ and
    \begin{equation*}
        R(T)\leq\left(\frac{128C_z\log\log T}{\log\frac{d+2}{d+1-d_z}}+16\right)\cdot T^\frac{d_z+1}{d_z+2}(\log T)^{\frac{1}{d_z+2}}.
    \end{equation*}
    
    The above analysis implies that we can achieve the optimal regret rate $\widetilde{\mathcal{O}}\left(T^\frac{d_z+1}{d_z+2}\right)$ by letting the \emph{for-loop} run $B^*$ times and finishing the remaining rounds in the \emph{Cleanup} step. In other words, $B^*+1$ rounds of communications are sufficient for A-BLiN to achieve the regret bound (\ref{up_bound}).
\end{proof}

\begin{remark}\label{remark:race}
    The quantity $\frac{r_m}{r_{m+1}}$ in line 9 of Algorithm \ref{alg_Op} may not be integers for some $m$. Thus, in practice we denote $\alpha_n=\lfloor\sum_{i=1}^nc_i\rfloor$, $\beta_n=\lceil\sum_{i=1}^nc_i\rceil$, and define rounded ACE Sequence $\{\widetilde{r}_m\}_{m\in\mathbb{N}}$ by $\widetilde{r}_m=2^{-\alpha_k}$ for $m=2k-1$ and $\widetilde{r}_m=2^{-\beta_k}$ for $m=2k$. Then the total regret can be divided as $R(T)=\sum_{1\leq k\leq B^*}R_{2k-1}+\sum_{1\leq k\leq B^*}R_{2k}+\sum_{m>2B^*}R_m$. For the first part we have $\widetilde{r}_{2k-2}\leq r_{k-1}$ and $\widetilde{r}_{2k-1}\geq r_k$, while for the second part we have $\frac{\widetilde{r}_{2k-1}}{\widetilde{r}_{2k}}=2$. Therefore, by similar argument to the proof of Theorem \ref{main_t_opt}, we can bound these three parts separately, and conclude that BLiN with rounded ACE sequence achieves the optimal regret bound $\widetilde{\mathcal{O}}(T^{\frac{d_z+1}{d_z+2}})$ by using only $\mathcal{O}(\log\log T)$ rounds of communications. The Rounded ACE Sequence is further investigated in \cite{feng2023lipschitz, han2023random}.
\end{remark} 

For rounded ACE Sequence, the quantity $\frac{\wtr_m}{\wtr_{m+1}}$ is an integer for any $m$, so the partition in Line 9 of Algorithm \ref{alg_Op} is well-defined. In Theorem \ref{t_race}, we show that BLiN with rounded ACE sequence can also achieve the optimal regret bound by using $\mathcal{O}(\log\log T)$ batches. For any $m$, if there exists $m'<m$ such that $\wtr_m \geq \wtr_{m'}$ (including the case where $ \frac{ \wtr_{2k-1} }{ \wtr_{2k} } = 1 $), then we skip the $m$-th batch. It is easy to verify that the following analysis is still valid in this case.

\begin{theorem}\label{t_race}
    With probability exceeding $1-\frac{2}{T^6}$, the $T$-step total regret $R(T)$ of Algorithm \ref{alg_Op} with rounded ACE sequence satisfies
    \begin{align*}
        R(T)
        \leq
        \left(\frac{128C_z\log\log T}{\log\frac{d+2}{d+1-d_z}}+512C_z+16\right) T^\frac{d_z+1}{d_z+2}(\log T)^{\frac{1}{d_z+2}}, 
    \end{align*} 
    where $d_z$ is the zooming dimension of the problem instance. In addition, Algorithm \ref{alg_Op} only needs no more than $\frac{2\log\log T}{\log \frac{d+2}{d+1-d_z}}+1$ rounds of communications to achieve this regret rate. 
\end{theorem}

\begin{proof}
    The proof of Theorem \ref{t_race} is similar to that of Theorem \ref{main_t_opt}.
    
    Firstly, we fix positive number $B^* = \frac{\log\log \frac{T}{\log T}-\log(d_z+2)}{\log \frac{d+2}{d+1-d_z}}$ and consider the first $2B^*$ batches. As is summairzed in Remark \ref{remark:race}, we bound the regret caused by the first $2B^*$ batches through two different arguments.
    
    For $m=2k-1$, $1\leq k\leq B^*$, we have $\wtr_m=2^{-\alpha_k}$ and $\wtr_{m-1}=2^{-\beta_{k-1}}$, and thus
    \begin{equation}\label{r_vs_wtr}
        \wtr_m\geq r_k\quad\text{and}\quad\wtr_{m-1}\leq r_{k-1}.
    \end{equation}
    Let $\mathcal{A}_m^+$ be set of cubes not eliminated in round $m$. Similar argument to Theorem \ref{main_t_opt} shows that $|\mathcal{A}_{m-1}^+| \le C_z\wtr_{m-1}^{-d_z}$. The total regret of round $m$ is 
    \begin{align*}
        R_m &=\sum_{C\in\mathcal{A}_m}\sum_{i=1}^{n_m}\Delta_{x_{C,i}}\\
        &\leq |\mathcal{A}_{m}| \cdot \frac{16\log T}{\wtr_m^2} \cdot 8\wtr_{m-1}\\
        &= \left(\frac{\wtr_{m-1}}{\wtr_m}\right)^d|\mathcal{A}_{m-1}^+|\cdot \frac{16\log T}{\wtr_m^2} \cdot 8\wtr_{m-1}\\
        &\leq \left(\frac{\wtr_{m-1}}{\wtr_m}\right)^d\cdot C_z\wtr_{m-1}^{-d_z}\cdot \frac{16\log T}{\wtr_m^2} \cdot 8\wtr_{m-1}\\
        &\leq \wtr_{m-1}^{d+1-d_z}\cdot\wtr_m^{-d-2}\cdot128C_z\log T \\
        &\leq r_{k-1}^{d+1-d_z}\cdot r_k^{-d-2}\cdot128C_z\log T\\
        &= T^\frac{d_z+1}{d_z+2}\cdot128C_z\cdot(\log T)^{\frac{1}{d_z+2}},
    \end{align*}
    where the sixth line follows from (\ref{r_vs_wtr}), and the seventh line follows from (\ref{eq:firstB}). Summing over $k$ gives that
    \begin{align}\label{r:odd}
        \sum_{k=1}^{B^*}R_{2k-1}
        \leq \;  
        T^\frac{d_z+1}{d_z+2}\cdot128C_z\cdot B^*\cdot(\log T)^{\frac{1}{d_z+2}}
        \leq \; 
        \frac{128C_z\log\log T}{\log\frac{d+2}{d+1-d_z}}\cdot T^\frac{d_z+1}{d_z+2}(\log T)^{\frac{1}{d_z+2}}. 
    \end{align}
    For $m=2k$, $1\leq k\leq B^*$, we have $\wtr_m=2^{-\beta_k}$ and $\wtr_{m-1}=2^{-\alpha_k}$, and thus $\wtr_m=\frac{1}{2}\wtr_{m-1}$. Lemma \ref{lem:eli} shows that any cube in $\mathcal{A}_m$ is a subset of $S(16\wtr_m)$, so we have $|\mathcal{A}_m|\leq N_{\wtr_m}\leq C_z\wtr_m^{-d_z}$. Therefore, the total regret of round $m$ is
    \begin{align*}
        R_m=\sum_{C\in\mathcal{A}_m}\sum_{i=1}^{n_m}\Delta_{x_{C,i}}
        \leq|\mathcal{A}_m|\cdot\frac{16\log T}{\wtr_m^2}\cdot16\wtr_m
        \leq C_z\wtr_m^{-d_z}\cdot\frac{16\log T}{\wtr_m^2}\cdot16\wtr_m
        =\wtr_m^{-d_z-1}\cdot256C_z\log T.
    \end{align*}
    Since $\frac{\wtr_{2k-2}}{\wtr_{2k}}=\frac{\wtr_{2k-2}}{\wtr_{2k-1}}\cdot\frac{\wtr_{2k-1}}{\wtr_{2k}}\geq2$, summing over $k$ gives that 
    \begin{equation}\label{eq:r-even-fixB}
        \sum_{k=1}^{B^*}R_{2k}\leq\wtr_{2B^*}^{-d_z-1}\cdot512C_z\log T.
    \end{equation}
    The definition of round ACE Sequence shows that
    \begin{align*}
        \wtr_{2B^*}
        = \; 
        2^{-\lceil\sum_{i=1}^{B^*}c_i\rceil}=2^{-\left\lceil c_1\left(\frac{1-\eta^{B^*}}{1-\eta}\right)\right\rceil} 
        = \; 
        2^{-\left\lceil\frac{\log \frac{T}{\log T}}{d_z+2}-1\right\rceil}\geq \left(\frac{T}{\log T}\right)^{-\frac{1}{d_z+2}}, 
    \end{align*}
    so we have
    \begin{align}\label{r:even}
        \sum_{k=1}^{B^*}R_{2k}
        \leq \; 
        \left(\left(\frac{T}{\log T}\right)^{-\frac{1}{d_z+2}}\right)^{-d_z-1}\cdot512C_z\log T 
        = \; 
        T^\frac{d_z+1}{d_z+2}\cdot512C_z(\log T)^{\frac{1}{d_z+2}}.  
    \end{align} 
    Similar argument to Theorem \ref{main_t_opt} shows that the total regret after $2B^*$ batches is upper bounded by
    $8\wtr_{2B^*}T$. Since 
    \begin{align*}
        \wtr_{2B^*}=2^{-\left\lceil\frac{\log \frac{T}{\log T}}{d_z+2}-1\right\rceil}\leq2^{-\frac{\log \frac{T}{\log T}}{d_z+2}+1}\leq2\left(\frac{T}{\log T}\right)^{-\frac{1}{d_z+2}},
    \end{align*}
    we further have
    \begin{equation}\label{r:last}
        \sum_{m>2B^*}R_m\leq8\wtr_{2B^*}T\leq16\cdot T^\frac{d_z+1}{d_z+2}(\log T)^{\frac{1}{d_z+2}}.
    \end{equation}
    
    Combining (\ref{r:odd}), (\ref{r:even}) and (\ref{r:last}), we conclude that
    \begin{align*}
        R(T)
        \leq\left(\frac{128C_z\log\log T}{\log\frac{d+2}{d+1-d_z}}+512C_z+16\right)\cdot T^\frac{d_z+1}{d_z+2}(\log T)^{\frac{1}{d_z+2}}.
    \end{align*}
    The analysis in Theorem \ref{t_race} implies that we can achieve the optimal regret rate $\widetilde{\mathcal{O}}\left(T^\frac{d_z+1}{d_z+2}\right)$ by letting the \emph{for-loop} of Algorithm \ref{alg_Op} run $2B^*$ times and finishing the remaining rounds in the \emph{Cleanup} step. In other words, $2B^*+1$ rounds of communications are sufficient for BLiN to achieve the optimal regret.
\end{proof}

\begin{remark}
    The proof of Theorem \ref{t_race} implies a regret upper bound of BLiN with a fixed number of batches $B$. See Appendix \ref{app:fix_B} for the detailed analysis.
\end{remark}

\section{Lower Bounds}\label{s_lb}


In this section, we present lower bounds for Lipschitz bandits with batched feedback, which in turn gives communication lower bounds for all Lipschitz bandit algorithms. Our lower bounds depend on the rounds of communications $B$. When $B$ is sufficiently large, our results match the upper bound for the vanilla Lipschitz bandit problem $ \widetilde{O} \left(T^{\frac{d_z+1}{d_z+2}}\right) $. More importantly, this dependency on $B$ gives the minimal rounds of communications needed to achieve optimal regret bound for all Lipschitz bandit algorithms, which is summarized in Corollary \ref{coro}. Since this lower bound matches the upper bound presented in Theorem \ref{main_t_opt}, BLiN optimally solves Lipschitz bandits with minimal communication. 

Similar to most lower bound proofs, we need to construct problem instances that are difficult to differentiate. What's different is that we need to carefully integrate batched feedback pattern \cite{perchet2016batched} with the Lipschitz payoff reward \cite{slivkins2011contextual,lu2019optimal}. To capture the adaptivity in grid determination, we construct ``static reference communication grids'' to remove the stochasticity in grid selection \cite{agarwal2017learning,gao2019batched}. Moreover, to prove the lower bounds for general $d_z\leq d$, we apply a ``linear-decaying extension'' technique to transfer instances from the $d_z$-dimensional subspace to the $d$-dimensional whole space.

The lower bound analysis is organized as follows. In Section \ref{sec:lb-d}, we present lower bounds for the full-dimensional case, that is, $d=d_z$. In Section \ref{sec:lb-d-static}, we consider the static grid case, where the grid is predetermined. This static grid case will provide intuition for the adaptive and more general case. In Section \ref{sec:lb-d-adaptive}, we provide the lower bound for general adaptive grid. Finally, in Section \ref{sec:lb-dz}, we apply the ``linear-decaying extension'' technique to prove lower bounds for the case that $d_z\leq d$. 


\subsection{Lower Bounds for the Full-Dimension Case}\label{sec:lb-d}
In this section, we let the zooming dimension $d_z$ equal to the ambient dimension $d$. The aim of considering this case first is to simplify the construction, and highlight the technique to deal with the batched feedback setting.

\subsubsection{The Static Grid Case}\label{sec:lb-d-static}

We first provide the lower bound for the case where the grid is static and determined before the game. 

The expected reward functions of hard instances are constructed as follows: we choose some `positions' and `heights', such that the expected reward function obtains local maximum of the specified `height' at the specified `position'. We will use the word `peak' to refer to the local maxima. 
The following theorem presents the lower bound for the static grid case. 


\begin{theorem}
    \label{st_lb} 
    Consider Lipschitz bandit problems with time horizon $T$ and ambient dimension $d$ such that the grid of reward communication $\mathcal{T}$ is static and determined before the game. If $B$ rounds of communications are allowed, then for any policy $\pi$, there exists a problem instance such that 
    \begin{align*}
        \E \[ R_T(\pi) \]
        \geq
        \frac{1}{32e^\frac{1}{16}}\cdot T^{\frac{1-\frac{1}{d+2}}{1-\left(\frac{1}{d+2}\right)^B}}.
    \end{align*}
\end{theorem} 

To prove Theorem \ref{st_lb}, we first show that for any $k>1$ there exists an instance such that $\E[R_T(\pi)]\geq\frac{t_k}{t_{k-1}^\frac{1}{d+2}}$. Fixing $k>1$, we let $r_k = \frac{1}{t_{k-1}^\frac{1}{d+2}}$ and $M_k := t_{k-1}r_k^2=\frac{1}{r_k^d}$.
Then we construct a set of problem instances $\mathcal{I}_k=\left\{I_{k,1},\cdots,I_{k,M_k}\right\}$, such that the gap between the highest peak and the second highest peak is about $r_k$ for every instance in $\mathcal{I}_k$.

 
Based on this construction, we prove that no algorithm can distinguish instances in $\mathcal{I}_k$ from one another in the first $(k-1)$ batches, so the worst-case regret is at least $r_kt_k$, which gives the inequality we need. For the first batch $(0, t_1]$, we can easily construct a set of instances where the worst-case regret is at least $t_1$, since no information is available during this time. Thus, there exists a problem instance such that
\begin{align*}
    \E[R_T(\pi)]\gtrsim\max\left\{t_1,\frac{t_2}{t_1^\frac{1}{d+2}},\cdots,\frac{t_B}{t_{B-1}^\frac{1}{d+2}}\right\} . 
\end{align*}
Since $0<t_1<\cdots<t_B=T$, the inequality in Theorem \ref{st_lb} follows. 

\begin{proof}[Proof of Theorem \ref{st_lb}]
    
    Fixing an index $k>1$, we first show that there exists an instance such that $\E[R_T(\pi)]\geq\frac{t_k}{t_{k-1}^\frac{1}{d+2}}$. We construct a set of problem instances that is difficult to distinguish. Let $r_k = \frac{1}{t_{k-1}^\frac{1}{d+2}}$ and $ M_k := t_{k-1}r_k^2=\frac{1}{r_k^d}$. We define $r_k$ and $M_k$ in this way to: 1. ensure that we can find a set of arms  $\mathcal{U}_k=\left\{u_{k,1},\cdots,u_{k,M_k}\right\}$ such that $d_\mathcal{A}(u_{k,i},u_{k,j})\geq r_k$ for any $i\neq j$; 2. maximize $r_k$ while ensuring $r_k\leq\sqrt{M_k/t_{k-1}}$, and thus the hard instances with maximum reward $r_k$ can still confuse a learner who only make $t_{k-1}$ observations. Then we consider a set of problem instances $\mathcal{I}_k=\left\{I_{k,1},\cdots,I_{k,M_k}\right\}$. The expected reward for $I_{k,1}$ is defined as
    \begin{align}
     \mu_{k,1}(x)=
    	\begin{cases} 
    		\frac{3}{4}r_k, \; \text{if} \; x=u_{k,1},\\
    		\frac{5}{8}r_k, \; \text{if} \; x=u_{k,j},\;j\neq1,\\
    		\max\left\{\frac{r_k}{2},\max_{u\in\mathcal{U}_k}\left\{\mu_{k,1}(u)-d_{\mathcal{A}}(x,u)\right\}\right\},\\
                \quad \text{if}\;x\in\mathcal{A}\setminus\mathcal{U}_k.
    	\end{cases} \label{eq:constuct1}
    \end{align}
    For $2\leq i\leq M_k$, the expected reward for $I_{k,i}$ is defined as
    \begin{align}
        \mu_{k,i} (x)=  
        \begin{cases} 
        	\frac{3}{4}r_k,\; \text{if} \; x=u_{k,1},\\
        	\frac{7}{8}r_k, \; \text{if} \; x=u_{k,i},\\
        	\frac{5}{8}r_k,\; \text{if} \; x=u_{k,j},\;j\neq1\;\text{and}\;j\neq i,\\
        	\max\left\{\frac{r_k}{2},\max_{u\in\mathcal{U}_k}\left\{\mu_{k,i} (u)-d_{\mathcal{A}}(x,u)\right\}\right\},\\
                \quad \text{if}\;x\in\mathcal{A}\setminus\mathcal{U}_k.
        \end{cases} \label{eq:constuct2}
    \end{align}

    Let $S_{k,i}=\B(u_{k,i},\frac{3}{8}r_k)$ (the ball with center $u_{k,i}$ and radius $\frac{3}{8}r_k$). It is easy to verify the following properties of construction (\ref{eq:constuct1}) and (\ref{eq:constuct2}):
    \begin{enumerate}
        \item For any $2\leq i\leq M_k$, $\mu_{k,i}(x)=\mu_{k,1}(x)$ for any $x\in\mathcal{A}\setminus S_{k,i}$;
        \item For any $2\leq i\leq M_k$, $\mu_{k,1}(x)\leq\mu_{k,i}(x)\leq\mu_{k,1}(x)+\frac{r_k}{4}$, for any $x\in S_{k,i}$;
        \item For any $1\leq i\leq M_k$, under $I_{k,i}$, pulling an arm that is not in $S_{k,i}$ incurs a regret at least $\frac{r_k}{8}$. 
    \end{enumerate}
    
    For all arm pulls in all problem instances, a Gaussian noise sampled from $\mathcal{N}(0,1)$ is added to the observed reward. This noise corruption is independent from all other randomness. 
    
    The lower bound of expected regret relies on the following lemma.
    \begin{lemma}\label{lem:kld}
    	For any policy $\pi$, there exists a problem instance $I\in\mathcal{I}_k$ such that
    	\begin{equation*}
    	    \E \[ R_T(\pi) \] \geq  \frac{r_k}{32}\cdot\sum_{j=1}^B \( t_j-t_{j-1} \) \exp\left\{-\frac{t_{j-1}r_k^2}{32(M_k-1)}\right\}.
    	\end{equation*}
    	 
    \end{lemma} 

    \begin{proof}
        
        Let $ x_t $ denote the choices of policy $\pi$ at time $t$, and $y_t$ denote the reward. Additionally, for $t_{j-1}<t\leq t_j$, we define $\Pr_{k.i}^t$ as the distribution of sequence $\left(x_1,y_1,\cdots,x_{t_{j-1}},y_{t_{j-1}}\right)$ under instance $I_{k,i}$ and policy $\pi$. It holds that 
        \begin{align}\label{eq:spl} 
            \sup_{I\in\mathcal{I}_k} \E R_T(\pi)  
            \geq
            \frac{1}{M_k}\sum_{i=1}^{M_k} \E_{\Pr_{k,i}} \left[ R_T(\pi) \right]
            \geq
            \frac{1}{M_k}\sum_{i=1}^{M_k} \sum_{t=1}^T\E_{\Pr_{k,i}^t} \left[ R^t(\pi) \right] 
            \geq
            \frac{r_k}{8}\sum_{t=1}^T\frac{1}{M_k}\sum_{i=1}^{M_k} \Pr_{k,i}^t(x_t\notin S_{k,i}),
        \end{align} 
        where $R^t(\pi)$ denotes the regret incurred by policy $\pi$ at time $t$. 
            
        From our construction, it is easy to see that $S_{k,i}\cap S_{k,j}=\varnothing$ for any $i\neq j$, so we can construct a test $\Psi$ such that $x_t\in S_{k,i}$ implies $\Psi= i$.  
        Then from Lemma \ref{lem:test-tree}, 
        \begin{align*}
             \; \frac{1}{M_k}\sum_{i=1}^{M_k} \Pr_{k,i}^t(x_t\notin S_{k,i})
            \geq \;  
            \frac{1}{M_k}\sum_{i=1}^{M_k} \Pr_{k,i}^t(\Psi\neq i) 
            \geq \;  
            \frac{1}{2M_k}\sum_{i=2}^{M_k} \exp 
            \left\{-D_{KL}\left(\Pr_{k,1}^t\|\Pr_{k,i}^t\right)\right\}.
        \end{align*} 

        To avoid notational clutter, for any $s,s'$ ($s \ge s'$), define 
        \begin{align*}
            (\x, \y )_{:s}^{:s'} =& \; (x_1, y_1,\cdots,x_{s'}, y_{s'}, \cdots ,x_s) . 
        \end{align*}
        
        Now we calculate $D_{KL}\left(\Pr_{k,1}^t\|\Pr_{k,i}^t\right)$. From the chain rule of KL-Divergence, we have
        \begin{align}
            & \; D_{KL}\left(\Pr_{k,1}^t\|\Pr_{k,i}^t\right) \nonumber \\
            =& \;
            D_{KL}\left(\Pr_{k,1}^t \( (\x ,\y)_{:t_{j-1}}^{:t_{j-1}} \) \| \Pr_{k,i}^t \( (\x ,\y)_{:t_{j-1}}^{:t_{j-1}} \) \right)\nonumber\\
            =& \;
            D_{KL}\left(\Pr_{k,1}^t \(  (\x ,\y)_{:t_{j-1}}^{:t_{j-1} - 1} \)  \| \Pr_{k,i}^t \( (\x ,\y)_{:t_{j-1}}^{:t_{j-1} - 1} \) \right) 
            +\E_{\Pr_{k,1}}\left(D_{KL}\left(\Pr_{k,1}^t(y_{t_{j-1}}|x_{t_{j-1}}) \| \Pr_{k,i}^t(y_{t_{j-1}}|x_{t_{j-1}})\right)\right) \label{chain_eq1}\\
           =& \;
            D_{KL}\left(\Pr_{k,1}^t \(  (\x ,\y)_{:t_{j-1}-1}^{:t_{j-1} - 1} \)  \| \Pr_{k,i}^t \( (\x ,\y)_{:t_{j-1}-1}^{:t_{j-1} - 1} \) \right)
            +\E_{\Pr_{k,1}}\left(D_{KL}\left(\Pr_{k,1}^t(y_{t_{j-1}}|x_{t_{j-1}}) \| \Pr_{k,i}^t(y_{t_{j-1}}|x_{t_{j-1}})\right)\right) \label{chain_eq1.5}\\
            \leq& \;
            D_{KL}\left(\Pr_{k,1}^t\( (\x ,\y)_{:t_{j-1}-1}^{:t_{j-1} - 1} \) \| \Pr_{k,i}^t\( (\x ,\y)_{:t_{j-1}-1}^{:t_{j-1} - 1} \)\right)
            +\E_{\Pr_{k,1}}\left(D_{KL}\left(N(\mu_{k,1}(x_{t_{j-1}}),1) \| N(\mu_{k,i}(x_{t_{j-1}}),1)\right)\right) \label{chain_eq2}\\
            =& \;
            D_{KL}\left(\Pr_{k,1}^t \( (\x ,\y)_{:t_{j-1}-1}^{:t_{j-1} - 1} \) \| \Pr_{k,i}^t \( (\x ,\y)_{:t_{j-1}-1}^{:t_{j-1} - 1} \)\right) 
            +\E_{\Pr_{k,1}}\left(\frac{1}{2}\left(\mu_{k,1}(x_{t_{j-1}})-\mu_{k,i}(x_{t_{j-1}})\right)^2\right) \nonumber\\
            \leq& \;
            D_{KL}\left(\Pr_{k,1}^t \( (\x ,\y)_{:t_{j-1}-1}^{:t_{j-1} - 1} \) \| \Pr_{k,i}^t \( (\x ,\y)_{:t_{j-1}-1}^{:t_{j-1} - 1} \)\right)
            +\E_{\Pr_{k,1}}\left(\bm{1}_{\{x_{t_{j-1}}\in S_{k,i}\}}\cdot\frac{1}{2}\left(\frac{r_k}{4}\right)^2\right)\label{chain_use_prop}\\
            =& \;
            D_{KL}\left(\Pr_{k,1}^t \( (\x ,\y)_{:t_{j-1}-1}^{:t_{j-1} - 1} \) \| \Pr_{k,i}^t \( (\x ,\y)_{:t_{j-1}-1}^{:t_{j-1} - 1} \)\right)
            +\frac{r_k^2}{32}\cdot\Pr_{k,1}\left(x_{t_{j-1}}\in S_{k,i}\right),\label{chain_singal}
        \end{align}
        where (\ref{chain_eq1}) uses chain rule for KL-divergence and the conditional independence of the reward, (\ref{chain_eq1.5}) removes dependence on $x_{t_{j - 1}}$ in the first term by another use of chain rule and the fact that the distribution of $x_{t_{j-1}}$ is fully determined by the policy and the distribution of $(\x ,\y)_{:t_{j-1}-1}^{:t_{j-1} - 1}$, (\ref{chain_eq2}) uses that the rewards are corrupted by a standard normal noise, and (\ref{chain_use_prop}) uses the first two properties of the construction.
        
        Since (\ref{chain_singal}) holds for all $t \le t_{j-1}$, we conclude that
        \begin{align}\label{chain}
            D_{KL}\left(\Pr_{k,1}^t\|\Pr_{k,i}^t\right)\leq\frac{r_k^2}{32}\sum_{s\leq t_{j-1}}\Pr_{k,1}\left(x_s\in S_{k,i}\right)=\frac{r_k^2}{32} \E_{\Pr_{k,1}}\tau_i,
        \end{align}
        where $\tau_i$ denotes the number of pulls of arms in $S_{k,i}$ before the batch containing $t$. Then for all $t \in ( t_{j-1}, t_j]$, we have  
        \begin{align}
            \frac{1}{M_k}\sum_{i=1}^{M_k} \Pr_{k,i}^t(x_t\notin S_{k,i})
            \geq& \; 
            \frac{1}{2M_k}\sum_{i=2}^{M_k} \exp\left\{-\frac{r_k^2}{32} \E_{\Pr_{k,1}}\tau_i\right\} \nonumber \\ 
            \geq& \;  
            \frac{M_k-1}{2M_k}\exp\left\{-\frac{r_k^2}{32(M_k-1)}\sum_{i=2}^{M_k}\E_{\Pr_{k,1}} \tau_i \right\} \label{eq:use-jensen}  \\
            \geq& \; 
            \frac{1}{4}\exp\left\{-\frac{r_k^2t_{j-1}}{32(M_k-1)}\right\}, \label{eq:hyp}
        \end{align} 
        where (\ref{eq:use-jensen}) uses the Jensen' inequality, and (\ref{eq:hyp}) uses the fact that $\sum_{i=2}^{M_k}\tau_i\leq t_{j-1}$. Finally, we substitute (\ref{eq:hyp}) to (\ref{eq:spl}) to finish the proof. 
    \end{proof} 
    
	Since $M_k=t_{k-1}r_k^2$, the expected regret of policy $\pi$ satisfies
	\begin{align*} 
		\E \[ R_T(\pi) \] 
		&\geq 
		\frac{r_k}{32} \cdot\sum_{j=1}^B \( t_j-t_{j-1} \) \exp\left\{-\frac{t_{j-1}r_k^2}{32(M_k-1)}\right\}\\ 
		&\geq
		\frac{r_k}{32} \cdot\sum_{j=1}^B \( t_j-t_{j-1} \) \exp\left\{-\frac{t_{j-1}r_k^2}{16M_k}\right\}\\
		&\geq 
		\frac{r_k}{32} \cdot\sum_{j=1}^B \( t_j-t_{j-1}\) \exp\left\{-\frac{t_{j-1}}{16t_{k-1}}\right\}
	\end{align*}
	on instance $I$ defined in Lemma \ref{lem:kld}.
	
	By omitting terms with $j > k$ in the above summation, we have 
	\begin{align*}
	    \E[R_T(\pi)]&\geq\frac{r_k}{32} \cdot \sum_{j=1}^B \( t_j-t_{j-1}\) \exp\left\{-\frac{t_{j-1}}{16t_{k-1}}\right\}  \\
	    &\ge 
	    \frac{r_k}{32} \cdot \sum_{j=1}^{k} \( t_j-t_{j-1}\) \exp\left\{-\frac{1}{16}\right\}\\&= \frac{1}{32e^\frac{1}{16}}r_kt_k= \frac{1}{32e^\frac{1}{16}}\cdot\frac{t_k}{t_{k-1}^\frac{1}{d+2}}.
	\end{align*}
   	The above analysis can be applied for any $k>1$. For the first batch $(0, t_1]$, we can easily construct a set of instances where the worst-case regret is at least $t_1$, since no information is available during this time.  Thus, there exists a problem instance such that
   	\begin{align*}
   	    \E[R_T(\pi)]\geq\frac{1}{32e^\frac{1}{16}}\max\left\{t_1,\frac{t_2}{t_1^\frac{1}{d+2}},\cdots,\frac{t_B}{t_{B-1}^\frac{1}{d+2}}\right\}.
   	\end{align*}
    
    Since $0<t_1<\cdots<t_B=T$, we further have
    \begin{align*}
        \max\left\{t_1,\frac{t_2}{t_1^\varepsilon},\cdots,\frac{t_B}{t_{B-1}^\varepsilon}\right\}
        \geq\left(t_1^{\varepsilon^{B-1}}\cdot\left(\frac{t_2}{t_1^\varepsilon}\right)^{\varepsilon^{B-2}}\cdots\left(\frac{t_{B-1}}{t_{B-2}^\varepsilon}\right)^{\varepsilon}\cdot\frac{t_B}{t_{B-1}^\varepsilon}\right)^{\frac{1}{\sum_{i=1}^{B-1}\varepsilon^i}}
        =T^{\frac{1-\varepsilon}{1-\varepsilon^B}},
    \end{align*}
    where $\varepsilon=\frac{1}{d+2}$. The above two inequalities imply that
   	\begin{align*}
	    \E \[ R_T(\pi) \]
	    \geq
	    \frac{1}{32e^\frac{1}{16}}\cdot T^{\frac{1-\frac{1}{d+2}}{1-\left(\frac{1}{d+2}\right)^B}},
	\end{align*}
    which finishes the proof. 
\end{proof}

\subsubsection{Removing the Static Grid Assumption}\label{sec:lb-d-adaptive}
So far we have derived the lower bound for the static grid case. Yet there is a gap between the static and the adaptive case. We will close this gap in the following Theorem.

    

\begin{theorem}\label{thm:ada_lb}
    Consider Lipschitz bandit problems with time horizon $T$ and ambient dimension $d$ such that the grid of reward communication $\mathcal{T}$ is adaptively determined by the player. If $B$ rounds of communications are allowed, then for any policy $\pi$, there exists a problem instance such that 
    \begin{align*}
        \E\left[ R_T(\pi) \right]\geq\frac{1}{256B^2}T^{\frac{1-\frac{1}{d+2}}{1-\left(\frac{1}{d+2}\right)^B}}.
    \end{align*}
\end{theorem}

To prove Theorem \ref{thm:ada_lb}, we consider a reference static grid $\mathcal{T}_r=\{T_0,T_1,\cdots,T_B\}$, where $T_j=T^\frac{1-\varepsilon^j}{1-\varepsilon^B}$ for $\varepsilon=\frac{1}{d+2}$. 
We set the reference grid in this way because it is the solution to the following optimization problem
\begin{equation*}
    \max_{1\leq T_1\leq\cdots\leq T_B=T}\min\left\{T_1,\frac{T_2}{T_1^\varepsilon},\cdots,\frac{T_B}{T_{B-1}^\varepsilon}\right\}.
\end{equation*}
Then we construct a series of `worlds', denoted by $\mathcal{I}_1,\cdots,\mathcal{I}_B$. Each world is a set of problem instances, and each problem instance in world $\mathcal{I}_j$ is defined by peak location set $\mathcal{U}_j$ and basic height $r_j$, where the sets $\mathcal{U}_j$ and quantities $r_j$ for $1\leq j\leq B$ are presented in the proof below.
    
Based on these constructions, we first prove that for any adaptive grid and policy, there exists an index $j$ such that the event $A_j=\{t_{j-1}<T_{j-1},\;t_j\geq T_j\}$ happens with sufficiently high probability in world $\mathcal{I}_j$. Then similar to Theorem \ref{st_lb}, we prove that in world $\mathcal{I}_j$ there exists a set of problem instances that is difficult to differentiate in the first $j-1$ batches. In addition, event $A_j$ implies that $t_j\geq T_j$, so the worst-case regret is at least $r_jT_j$, which gives the lower bound we need.

\begin{proof}[Proof of Theorem \ref{thm:ada_lb}]
    
    
    Firstly, we define $r_j=\frac{1}{T_{j-1}^\varepsilon B}$, where $\varepsilon=\frac{1}{d+2}$, and define $M_j=\frac{1}{r_j^d}$.
    From the definition, we have
    \begin{align} 
        T_{j-1}r_j^2=\frac{1}{r_j^d B^{d+2}}\leq\frac{1}{r_j^dB^2}=\frac{M_j}{B^2}. \label{eq:rm-relation}
    \end{align}
    For $1\leq j\leq B$, we can find sets of arms $\mathcal{U}_j=\{u_{j,1},\cdots,u_{j,M_j}\}$ such that (a) $d_{\mathcal{A}}(u_{j,m},u_{j,n})\geq r_j$ for any $m\neq n$, and (b) $u_{1,M_1}=\cdots=u_{B,M_B}$. 
    
     Then we present the construction of worlds $\mathcal{I}_1,\cdots,\mathcal{I}_B$. For $1\leq j\leq B-1$, we let $\mathcal{I}_j=\{I_{j,k}\}_{k=1}^{M_j-1}$, and the expected reward of $I_{j,k}$ is defined as
    \begin{align}
    	\mu_{j,k}(x)=\left\{
    	\begin{aligned}
    		&\frac{r_1}{2}+\frac{r_j}{16}+\frac{r_B}{16},&x=u_{j,k},\\
    		&\frac{r_1}{2}+\frac{r_B}{16},&x=u_{j,M_j},
    	\end{aligned}\right. \label{eq:instance-ada1}
    \end{align}
    and $\mu_{j,k}(x)=\max\left\{\frac{r_1}{2},\max_{u\in\mathcal{U}_j}\left\{\mu_{j,k}(u)-d_\mathcal{A}(x,u)\right\}\right\}$, otherwise.
    For $j=B$, we let $\mathcal{I}_B=\{I_B\}$. The expected reward of $I_B$ is defined as
    \begin{align}
    	\mu_B (u_{B,M_B})=\frac{r_1}{2}+\frac{r_B}{16} \label{eq:instance-ada2}
    \end{align}
    and $\mu_B(x)=\max\left\{\frac{r_1}{2},\mu_B(u_{B,M_B})-d_{\mathcal{A}}(x,u_{B,M_B})\right\}$, otherwise. Roughly speaking, our constructions satisfy two properties: for each $j\neq B$ and $1\leq k\leq M_j-1$, 
    \begin{enumerate}
        \item $\mu_{j,k}$ is close to $\mu_B$;
        \item under $I_{j,k}$, pulling an arm that is far from $u_{j,k}$ incurs a regret at least $\frac{r_j}{16}$.
    \end{enumerate}
    The formal version of these two properties are presented in the proof of Lemma \ref{exist_pj} and Lemma \ref{adagrid_lemma}. 
    
    As mentioned above, based on these constructions, we first show that for any adaptive grid $ \mathcal{T} = \{t_0,\cdots,t_B\} $, there exists an index $j$ such that $(t_{j-1},t_j]$ is sufficiently large in world $\mathcal{I}_j$. More formally, for each $j\in[B]$, and event $A_j=\{t_{j-1}<T_{j-1},\;t_j\geq T_j\}$, we define the quantities $p_j := \frac{1}{M_j-1}\sum_{k=1}^{M_j-1}\Pr_{j,k}(A_j)$ for $j\leq B-1$ and $p_B := \Pr_B(A_B)$,
    where $\Pr_{j,k}(A_j)$ denotes the probability of the event $A_j$ under instance $I_{j,k}$ and policy $\pi$. For these quantities, we have the following lemma.

\begin{lemma}
    \label{exist_pj} 
    For any adaptive grid $\mathcal{T}$ and policy $\pi$, it holds that $\sum_{j=1}^B p_j\geq\frac{7}{8}.$  
\end{lemma} 

\begin{proof}
    For $1\leq j\leq B-1$ and $1\leq k\leq M_j-1$, we define $S_{j,k}=\B(u_{j,k},\frac{3}{8}r_j)$, which is the ball centered as $ u_{j,k} $ with radius $ \frac{3}{8}r_j $. It is easy to verify the following properties of our construction (\ref{eq:instance-ada1}) and (\ref{eq:instance-ada2}):
    \begin{enumerate}
    	\item $\mu_{j,k}(x)=\mu_B (x)$ for any $x\notin S_{j,k}$;
    	\item $\mu_B (x) \leq \mu_{j,k}(x)\leq\mu_B (x)+\frac{r_j}{8}$, for any $x\in S_{j,k}$.
    \end{enumerate}
    Let $ x_t $ denote the choices of policy $\pi$ at time $t$, and $y_t$ denote the reward. For $t_{j-1}<t\leq t_j$, we define $\Pr_{j,k}^t$ (resp. $\Pr_{B}^{t}$) as the distribution of sequence $\left(x_1,y_1,\cdots,x_{t_{j-1}},y_{t_{j-1}}\right)$ under instance $I_{j,k}$ (resp. $I_B$) and policy $\pi$. Since event $A_j$ can be completely described by the observations up to time $T_{j-1}$ ($A_j$ is an event in the $\sigma$-algebra where $\Pr_{j,k}^{T_{j-1}}$ and $\Pr_B^{T_{j-1}}$ are defined on), we can use the definition of total variation to get 
    \begin{align*}
    	|\Pr_B (A_j)-\Pr_{j,k}(A_j)| 
            = \;  
            | \Pr_B^{T_{j-1}} (A_j) - \Pr_{j,k}^{T_{j-1}} (A_j) | 
    	\leq\;  
            TV \(\Pr_B^{T_{j-1}},\Pr_{j,k}^{T_{j-1}} \) . 
    \end{align*}

    For the total variation, we apply Lemma \ref{lem:pinsker-type} to get 
    \begin{align*} 
    	 \; \frac{1}{M_j-1} \sum_{k=1}^{M_j-1} TV \( \Pr_B^{T_{j-1}},\Pr_{j,k}^{T_{j-1}}\) 
    	\leq \; 
    	\frac{1}{M_j-1} \sum_{k=1}^{M_j-1} \sqrt{1-\exp\(-D_{KL}\(\Pr_B^{T_{j-1}}\|\Pr_{j,k}^{T_{j-1}}\)\)}. 
    \end{align*} 
    An argument similar to (\ref{chain}) yields that
    \begin{align*}
    	D_{KL}\(\Pr_B^{T_{j-1}}\|\Pr_{j,k}^{T_{j-1}} \) \le \frac{r_j^2}{128} \E_{\Pr_B}\tau_k, 
    \end{align*}
    where $\tau_k$ denotes the number of pulls which is in $S_{j,k}$ before the batch containing $T_{j-1}$. Combining the above two inequalities gives
    \begin{align} 
    	\frac{1}{M_j-1} \sum_{k=1}^{M_j-1} TV \( \Pr_B^{T_{j-1}},\Pr_{j,k}^{T_{j-1}}\)
    	\leq& \;  
    	\frac{1}{M_j-1} \sum_{k=1}^{M_j-1}\sqrt{1- \exp\(-\frac{r_j^2}{128} \E_{\Pr_B}\tau_k\)} \nonumber \\ 
    	\leq& \;  
    	\sqrt{1-\exp\(-\frac{r_j^2}{128(M_j-1)}\E_{\Pr_B}\[\sum_{k=1}^{M_j-1}\tau_k\]\)} \label{l7_jenson} \\ 
    	\leq& \;  
    	\sqrt{  1-\exp\left(-\frac{r_j^2T_{j-1}}{128(M_j-1)}\right)} \label{l7_total} \\ 
    	\leq& \;  
    	\sqrt{ 1-\exp\left(-\frac{1}{64B^2}\right)} \label{l7_useeq} \\ 
    	\leq& \;  \frac{1}{8B}, \nonumber 
    \end{align}
    where (\ref{l7_jenson}) uses Jensen's inequality, (\ref{l7_total}) uses the fact that $\sum_{k=1}^{M_j-1} \tau_k\leq T_{j-1}$, and (\ref{l7_useeq}) uses (\ref{eq:rm-relation}).

    Plugging the above results implies that
    \begin{align*} 
        |\Pr_B(A_j) - p_j| \le \frac{1}{M_j -1 } \sum_{k=1}^{M_j-1} |\Pr_B (A_j)-\Pr_{j,k}(A_j)| \le  \frac{1}{8B}. 
    \end{align*}
    Since $\sum_{j=1}^B \Pr_B \(  A_j \) \ge \Pr_B \( \cup_{j=1}^B A_j \) = 1 $, it holds that 
    \begin{equation*} 
        \sum_{j=1}^B p_j \geq \Pr_B (A_B) +\sum_{j=1}^{B-1} \(\Pr_B(A_j)-\frac{1}{8B} \) \geq \frac{7}{8}.\qedhere
    \end{equation*} 
\end{proof}

Lemma \ref{exist_pj} implies that there exists some $j$ such that $p_j>\frac{7}{8B}$. Then similar to Theorem \ref{st_lb}, we show that the worst-case regret of the policy in world $\mathcal{I}_j$ gives the lower bound we need. 

\begin{lemma}\label{adagrid_lemma}
    For adaptive grid $\mathcal{T}$ and policy $\pi$, if index $j$ satisfies $p_j\geq\frac{7}{8B}$, then 
    there exists a problem instance $I$ such that 
    \begin{align*}
        \E\left[ R_T(\pi) \right]\geq\frac{1}{256B^2}T^{\frac{1-\frac{1}{d+2}}{1-\left(\frac{1}{d+2}\right)^B}}.
    \end{align*}
\end{lemma}

\begin{proof}
   Here we proceed with the case where $j \le B-1$. The case for $j = B$ can be proved analogously. 
   
   For any $1\leq k \leq M_j-1$, we construct a set of problem instances $\mathcal{I}_{j,k} = \left(I_{j,k,l} \right)_{1\leq l\leq M_j}$.
    For $l\neq k$, the expected reward of $I_{j,k,l}$ is defined as 
    \begin{align*}
    	\mu_{j,k,l}(x) 
            =
    	    \begin{cases}
    		\mu_{j,k}(x)+\frac{3r_j}{16}, \text{ if }x=u_{j,l},\\
    		\mu_{j,k}(x), \text{ if } x \in \mathcal{U}_j \text{ and }  x \neq u_{j,l} , \\ 
    		\max\left\{\frac{r_1}{2},\max_{u\in\mathcal{U}_j}\left\{\mu_{j,k,l}(u)-d_{\mathcal{A}}(x,u)\right\}\right\}, \text{ otherwise.}
    	\end{cases}
    \end{align*} 
    where $\mu_{j,k}$ is defined in (\ref{eq:instance-ada1}).
    For $l=k$, we let $\mu_{j,k,k}=\mu_{j,k}$. 
    
    We define $C_{j,k}= \B \(u_{j,k},\frac{r_j}{4} \)$, and our construction $\mathcal{I}_{j,k}$ has the following properties:
    \begin{enumerate}
        \item For any $l\neq k$, $\mu_{j,k,l}(x)=\mu_{j,k,k}(x)$ for any $x\notin C_{j,l}$;
        \item For any $l\neq k$, $\mu_{j,k,k}(x)\leq\mu_{j,k,l}(x)\leq\mu_{j,k,k}(x)+\frac{3r_j}{16}$ for any $x\in C_{j,l}$;
        \item For any $1\leq l\leq M_j$, under $I_{j,k,l}$, pulling an arm that is not in $C_{j,l}$ incurs a regret at least $\frac{r_j}{16}$.
    \end{enumerate}
    Let $x_t$ denote the choices of policy $\pi$ at time $t$, and $y_t$ denote the reward. For $t_{j-1}<t\leq t_j$, we define $\Pr_{j,k,l}^t$ as the distribution of sequence $\left(x_1,y_1,\cdots,x_{t_{j-1}},y_{t_{j-1}}\right)$ under instance $I_{j,k,l}$ and policy $\pi$. From similar argument in (\ref{eq:spl}), it holds that
    \begin{align}
        \sup_{I\in\mathcal{I}_{j,k}}\E \[ R_T(\pi) \] 
        \geq
        \frac{r_j}{16} \sum_{t=1}^T \frac{1}{M_j} \sum_{l=1}^{M_j} \Pr_{j,k,l}^t (x_t\notin C_{j,l}). \label{eq:lower-bound-ada1} 
    \end{align}
    From our construction, it is easy to see that $C_{j,k_1}\cap C_{j,k_2}=\varnothing$ for any $k_1\neq k_2$, so we can construct a test $\Psi$ such that $x_t\in C_{j,k}$ implies $\Psi= k$. By Lemma \ref{lem:test-tree} with a star graph on $[K]$ with center $k$, we have 
    \begin{align}
        \frac{1}{M_j} \sum_{l=1}^{M_j} \Pr_{j,k,l}^t(x_t\notin C_{j,l}) 
        \ge 
        \frac{1}{ M_j } \sum_{ l \neq k } \int \min \left\{d \Pr_{j,k,k}^t,d \Pr_{j,k,l}^t\right\} . \label{eq:lower-bound-ada2}
    \end{align} 
    
    Combining (\ref{eq:lower-bound-ada1}) and (\ref{eq:lower-bound-ada2}) gives
    \begin{align}
        \sup_{I\in\mathcal{I}_{j,k}} \E \[ R_T(\pi) \]
        \geq& \;  
        \frac{r_j}{16} \sum_{t=1}^T \frac{1}{M_j}\sum_{l\neq k}\int \min \left\{d \Pr_{j,k,k}^t,d \Pr_{j,k,l}^t\right\}\nonumber  \\ 
        \geq& \; 
        \frac{r_j}{16}\sum_{t=1}^{T_j}\frac{1}{M_j}\sum_{l\neq k}\int\min\left\{d \Pr_{j,k,k}^t,d \Pr_{j,k,l}^t\right\}\nonumber \\
        \geq& \;  
        \frac{r_jT_j}{16}\cdot\frac{1}{M_j}\sum_{l\neq k}\int\min\left\{d \Pr_{j,k,k}^{T_j},  d\Pr_{j,k,l}^{T_j}\right\}\label{eq:lower_cons_1_1} \\ 
        \geq& \;  
        \frac{r_jT_j}{16}\cdot\frac{1}{M_j}\sum_{l\neq k}\int_{A_j}\min\left\{ d \Pr_{j,k,k}^{T_j}, d \Pr_{j,k,l}^{T_j}\right\}\label{eq:lower_cons_1_2} \\ 
        \geq& \;  
        \frac{r_jT_j}{16}\cdot\frac{1}{M_j}\sum_{l\neq k}\int_{A_j}\min\left\{ d \Pr_{j,k,k}^{T_{j-1}}, d \Pr_{j,k,l}^{T_{j-1}}\right\}, \label{eq:lower_cons_1}
    \end{align}  
    where (\ref{eq:lower_cons_1_1}) follows from data processing inequality of total variation and the equation $\int\min\left\{dP,dQ\right\}=1-TV(P,Q)$, (\ref{eq:lower_cons_1_2}) restricts the integration to event $A_j$, and (\ref{eq:lower_cons_1}) holds because the observations at time $T_j$ are the same as those at time $T_{j-1}$ under event $A_j$.
    
    For the term $\int_{A_j}\min\left\{d\Pr_{j,k,k}^{T_{j-1}},d\Pr_{j,k,l}^{T_{j-1}}\right\}$, it holds that 
    \begin{align}
        \int_{A_j}\min\left\{d\Pr_{j,k,k}^{T_{j-1}},d\Pr_{j,k,l}^{T_{j-1}}\right\}
        =&\; \int_{A_j}\frac{d\Pr_{j,k,k}^{T_{j-1}}+d\Pr_{j,k,l}^{T_{j-1}}-\left|d\Pr_{j,k,k}^{T_{j-1}}-d\Pr_{j,k,l}^{T_{j-1}}\right|}{2}\nonumber\\ 
        =&\; 
        \frac{\Pr_{j,k,k}^{T_{j-1}}(A_j)+\Pr_{j,k,l}^{T_{j-1}}(A_j)}{2}-\frac{1}{2}\int_{A_j}\left|d\Pr_{j,k,k}^{T_{j-1}}-d\Pr_{j,k,l}^{T_{j-1}}\right|\nonumber\\
        \geq&\;  
        \(\Pr_{j,k,k}^{T_{j-1}}(A_j)-\frac{1}{2}TV \(\Pr_{j,k,k}^{T_{j-1}},\Pr_{j,k,l}^{T_{j-1}}\)\) -TV \(\Pr_{j,k,k}^{T_{j-1}},\Pr_{j,k,l}^{T_{j-1}}\)\label{eq:lower_cons_2_1}\\ 
        =&\; 
        \Pr_{j,k}(A_j)-\frac{3}{2}TV \( \Pr_{j,k,k}^{T_{j-1}},\Pr_{j,k,l}^{T_{j-1}} \), \label{eq:lower_cons_2} 
    \end{align} 
    where (\ref{eq:lower_cons_2_1}) uses the inequality $|\Pr(A)-\mathbb{Q}(A)|\leq TV(\Pr,\mathbb{Q})$, and (\ref{eq:lower_cons_2}) holds because $I_{j,k}=I_{j,k,k}$ and $A_j$ can be determined by the observations up to $T_{j-1}$.
    
    We use an argument similar to (\ref{chain}) to get 
    \begin{align*}
    	D_{KL}\(\Pr_{j,k,k}^{T_{j-1}}\|\Pr_{j,k,l}^{T_{j-1}}\)
    	\le
    	\frac{1}{2}\cdot\left(\frac{3r_j}{16}\right)^2\E_{\Pr_{j,k}}\tau_l
    	\le \frac{r_j^2}{32} \E_{\Pr_{j,k}}\tau_l,
    \end{align*}
    where $\tau_l$ denotes the number of pulls which is in $C_{j,l}$ before the batch of time $T_{j-1}$. Then from Lemma \ref{lem:pinsker-type}, we have
    \begin{align}
        \frac{1}{M_j}\sum_{l\neq k} TV \( \Pr_{j,k,k}^{T_{j-1}},\Pr_{j,k,l}^{T_{j-1}}\)
        \leq& \; 
        \frac{1}{M_j}\sum_{l\neq k}\sqrt{1-\exp\(-D_{KL}\(\Pr_{j,k,k}^{T_{j-1}}\|\Pr_{j,k,l}^{T_{j-1}}\)\)} \nonumber\\ 
        \leq& \; 
        \frac{1}{M_j}\sum_{l\neq  k}\sqrt{1-\exp\(-\frac{r_j^2}{32}\E_{\Pr_{j,k}}\tau_l\)} \nonumber\\ 
        \leq& \; 
        \frac{M_j-1}{M_j}\sqrt{1-\exp\(-\frac{r_j^2}{32(M_j-1)}\sum_{l\neq k}\E_{\Pr_{j,k}}\tau_l\)} \nonumber\\ 
        \leq& \; 
        \frac{M_j-1}{M_j}\sqrt{1-\exp\(-\frac{r_j^2 T_{j-1}}{32(M_j-1)}\)} \nonumber\\
        \leq& \; 
        \frac{M_j-1}{M_j}\sqrt{1-\exp\(-\frac{M_j}{32(M_j-1) B^2}\)} \label{eq:lower_cons_3_1}\\
        \leq& \; 
        \frac{M_j-1}{M_j}\sqrt{\frac{M_j}{32(M_j-1) B^2}}\nonumber\\
        \leq& \; 
        \frac{1}{4B},\label{eq:lower_cons_3} 
    \end{align}
    where (\ref{eq:lower_cons_3_1}) uses (\ref{eq:rm-relation}).
         
    Combining (\ref{eq:lower_cons_1}), (\ref{eq:lower_cons_2}) and  (\ref{eq:lower_cons_3}) yields that
    \begin{align*} 
        \sup_{I\in\mathcal{I}_{j,k}}\E \[R_T(\pi) \] 
        \geq 
        \frac{1}{16}  r_jT_j\(\frac{\Pr_{j,k}(A_j)}{2}-\frac{3}{8B}\) \geq 
        \frac{1}{16B} T^{\frac{1-\varepsilon}{1 -  \varepsilon^B}}\(\frac{\Pr_{j,k}(A_j)}{2} -  \frac{3}{8B}\),
    \end{align*}  
    where $\varepsilon=\frac{1}{d+2}$. This inequality holds for any $k \leq M_j-1$. Averaging over $k$ yields
    \begin{align*}
            \sup_{I\in\cup_{k\leq M_j-1}\mathcal{I}_{j,k}}\E \[ R_T(\pi) \]
            \geq& \; 
            \frac{1}{16B} T^{\frac{1-\varepsilon}{1-\varepsilon^B}} \( \frac{1}{2(M_j-1)}\sum_{k=1}^{M_j-1} \Pr_{j,k}(A_j)-\frac{3}{8B} \)\\ 
            \geq& \; 
            \frac{1}{16B} T^{\frac{1-\varepsilon}{1-\varepsilon^B}} \( \frac{7}{16B}-\frac{3}{8B} \)\\ 
            \geq& \;  
            \frac{1}{256B^2}T^{\frac{1-\varepsilon}{1-\varepsilon^B}}, 
    \end{align*}
    where the second inequality holds from $p_j\geq\frac{7}{8B}$. Hence, the proof of Lemma \ref{adagrid_lemma} is completed.
\end{proof}

Finally, combining the above two lemmas, we arrive at the lower bound in Theorem \ref{thm:ada_lb}.
\end{proof}

\subsection{Communication Lower bound for Lipschitz Bandits with Batched Feedback}\label{sec:lb-dz}

Based on the constructions for the full-dimension case, we are now ready to present the theoretical lower bound of batched Lipschitz bandits with $d_z\leq d$. For easy understanding, we start with the result for the static grid case. In this section we provide lower bounds for integer-valued $d_z$. In general, the zooming dimension is not necessarily an integer.

\begin{theorem} \label{thm:lower-static-dz}
    Consider Lipschitz bandit problems with time horizon $T$, ambient dimension $d$ and zooming dimension $d_z\leq d$ such that the grid of reward communication $\mathcal{T}$ is static and determined before the game. If $B$ rounds of communications are allowed, then for any policy $\pi$, there exists a problem instance with zooming dimension $d_z$ such that
    \begin{equation*}
        \E[R_T(\pi)]\geq \frac{1}{64e^\frac{1}{16}}\cdot T^{\frac{1-\frac{1}{d_z+2}}{1-\left(\frac{1}{d_z+2}\right)^B}}.
    \end{equation*}
\end{theorem}

The proof technique of Theorem \ref{thm:lower-static-dz} is similar to that of Theorem \ref{st_lb}, with the main difference being the construction of hard instances. To build instances satisfying the statement in Theorem \ref{thm:lower-static-dz}, we first construct reward functions in a $d_z$-dimensional subspace according to the argument in Theorem \ref{st_lb}, and then use a ``linear-decaying extension'' technique to transfer them to the $d$-dimensional whole space. Below, we present the detailed construction of the hard instances.

To begin with, we introduce some settings and notations. For a given $d$, we consider the whole space $\mathcal{A}=[0,1]^d$ and $d_{\mathcal{A}}$ being the induced metric of $\|\cdot\|_\infty$. $\mathcal{A}$ can be represented by a Cartsian product $[0,1]^d=[0,1]^{d_z}\times[0,1]^{d-d_z}$. Therefore, each element in $\A$ can be represented as $(\alpha,\beta)$, the concatenation of $\alpha\in[0,1]^{d_z}$ and $\beta\in[0,1]^{d-d_z}$. Additionally, we use $\mathcal{A}_{d_z}$ to denote the $d_z$-dimensional subspace $\{(\alpha,0_{d-d_z})|\alpha\in[0,1]^{d_z}\}$, where $0_{d-d_z}$ is the zero vector in $[0,1]^{d-d_z}$.

For any fixed $k$, let $r_k=\frac{1}{t_{k-1}^\frac{1}{d_z+2}}$ and $M_k=t_{k-1}r_k^2=\frac{1}{r_k^{d_z}}$. We can find a set of arms $\mathcal{U}_k=\{u_{k,1},\cdots,u_{k,M_k}\}\subset\mathcal{A}_{d_z}$ such that $d_{\mathcal{A}}(u_{k,i},u_{k,j})\geq r_k$ for any $i\neq j$. As stated above, we first construct a set of expected reward functions $\{\nu_{k,1},\cdots,\nu_{k,M_k}\}$ on $\A_{d_z}$, which are the same as (\ref{eq:constuct1}) and (\ref{eq:constuct2}) in the proof of Theorem \ref{st_lb}. We define $\nu_{k,1}$ as
    \begin{align}\label{eq:constuct1-dz}
    	\nu_{k,1} (z)
            = 
    	\begin{cases} 
    		\frac{3}{4}r_k, \; \text{if} \; z=u_{k,1},\\
    		\frac{5}{8}r_k, \; \text{if} \; z=u_{k,j},\;j\neq1,\\
    		\max\left\{\frac{r_k}{2},\max_{u\in\mathcal{U}_k}\left\{\nu_{k,1}(u)-d_{\mathcal{A}}(z,u)\right\}\right\},\\
            \quad \text{if}\;z\in\mathcal{A}_{d_z}\setminus\mathcal{U}_k.
    	\end{cases} 
    \end{align}
    For $2\leq i\leq M_k$, $\nu_{k,i}$ is defined as
    \begin{align}\label{eq:constuct2-dz}
        \nu_{k,i} (z)
        = 
        \begin{cases} 
        	\frac{3}{4}r_k,\; \text{if} \; z=u_{k,1},\\
        	\frac{7}{8}r_k, \; \text{if} \; z=u_{k,i},\\
        	\frac{5}{8}r_k,\; \text{if} \; z=u_{k,j},\;j\neq1\;\text{and}\;j\neq i,\\
        	\max\left\{\frac{r_k}{2},\max_{u\in\mathcal{U}_k}\left\{\nu_{k,i} (u)-d_{\mathcal{A}}(z,u)\right\}\right\},\\
            \quad \text{if}\;z\in\mathcal{A}_{d_z}\setminus\mathcal{U}_k.
        \end{cases} 
    \end{align}
    Based on $\{\nu_{k,i}\}_{i=1}^{M_k}$, we define a set of problem instances $\I_k=\{I_{k,1},\cdots,I_{k,M_k}\}$. For each $1\leq i\leq M_k$, the expected reward for $I_{k,i}$ is defined as
    \begin{align}\label{eq:constuct3-dz}
        \mu_{k,i}((\alpha,\beta))
        =\frac{1}{2}\cdot\min\left\{\nu_{k,i}((\alpha,0_{d-d_z})), \frac{7}{8}r_k-\|\beta\|_\infty\right\},
    \end{align}
    where $(\alpha,\beta)$ is the concatenation of $\alpha\in[0,1]^{d_z}$ and $\beta\in[0,1]^{d-d_z}$. Note that $\nu((\alpha,0_{d-d_z}))$ is well-defined since $(\alpha,0_{d-d_z})\in\mathcal{A}_{d_z}$.

    For all arm pulls in all problem instances, an Gaussian noise sampled from $\mathcal{N}(0,1)$ is added to the observed reward. This noise corruption is independent from all other randomness. 

    Now we show that for each $1\leq k\leq B$ and $1\leq i\leq M_k$, $\mu_{k,i}$ is $1$-Lipschitz and the zooming dimension equals to $d_z$. Firstly, for any $(\alpha_1,\beta_1)$ and $(\alpha_2,\beta_2)$, we have
    \begin{align*}
        \mu_{k,i}((\alpha_1,\beta_1))-\mu_{k,i}((\alpha_2,\beta_2))
        \leq&\mu_{k,i}((\alpha_1,\beta_1))-\mu_{k,i}((\alpha_1,\beta_2))
        +\mu_{k,i}((\alpha_1,\beta_2))-\mu_{k,i}((\alpha_2,\beta_2))\\
        \leq&\frac{1}{2}\|\beta_1-\beta_2\|_\infty
        +\frac{1}{2}\big(\nu_{k,i}((\alpha_1,0_{d-d_z}))-\nu_{k,i}((\alpha_2,0_{d-d_z}))\big)\\
        \leq&\frac{1}{2}(\|\beta_1-\beta_2\|_\infty+\|\alpha_1-\alpha_2\|_\infty)\\
        \leq&\|(\alpha_1-\alpha_2,\beta_1-\beta_2)\|_\infty.
    \end{align*}
    Therefore, $\mu_{k,i}$ is $1$-Lipschitz. Secondly, for any $r\geq r_k$, (\ref{eq:constuct3-dz}) yields that $S(16r)=[0,1]^{d_z}\times[0, 32r]^{d-d_z}$. As a consequence, we have $N_r=32^{d-d_z}r^{-d_z}$, and the zooming dimension equals to $d_z$.

    After presenting the new constructions, we show that similar argument to the full-dimension case gives the lower bound we need. The remaining proof of Theorem \ref{thm:lower-static-dz} is deferred to Appendix \ref{app:lower-dz-static}.

    Finally, we combine all techniques in above analysis to obtain the lower bound for general $d_z$ and adaptive grid, and thus prove Theorem \ref{thm:lower-adaptive-dz-intro}.

    \begin{theorem}
        \label{thm:lower-adaptive-dz}
        Consider Lipschitz bandit problems with time horizon $T$, ambient dimension $d$ and zooming dimension $d_z\leq d$ such that the grid of reward communication $\mathcal{T}$ is adaptively determined by the player. If $B$ rounds of communications are allowed, then for any policy $\pi$, there exists a problem instance with zooming dimension $d_z$ such that 
    	\begin{align*} 
    	    \E \left[ R_T(\pi) \right] 
    	    \geq \frac{1}{512B^2}T^{\frac{1-\frac{1}{d_z+2}}{1-\left(\frac{1}{d_z+2}\right)^B}}.
    	\end{align*}
    \end{theorem}

    The proof of Theorem \ref{thm:lower-adaptive-dz} is deferred to Appendix \ref{app:lower-dz-adaptive}.


\section{Experiments}\label{exp}
    In this section, we present numerical studies of A-BLiN. In the experiments, we use the arm space $\mathcal{A}=[0,1]^2$ and the expected reward function $\mu(x)=1 - \frac{1}{2}\|x-x_1\|_2 -\frac{3}{10}\|x-x_2\|_2$, where $x_1=(0.8,\;0.7)$ and $x_2=(0.1,\;0.1)$. The landscape of $\mu$ and the resulting partition is shown in Figure \ref{parti}. As can be seen, the partition is finer in the area closer to the optimal arm $x^*=(0.8,\;0.7)$.

    \begin{figure*}[htb]
        \centering
        \subfloat[Partition]{
            \includegraphics[height=5cm]{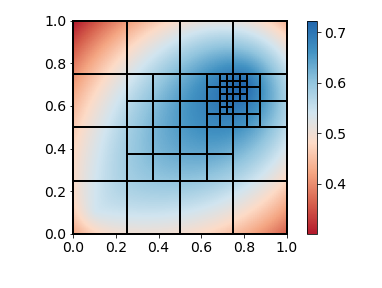}
            \label{parti}
        }
        \subfloat[Regret]{
            \includegraphics[height=5cm]{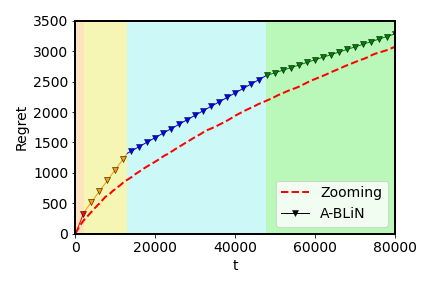}
            \label{regret}
        }
        \caption{Resulting partition and regret of A-BLiN. In Figure \ref{parti}, we show the resulting partition of A-BLiN. The background color denotes the true value of expected reward $\mu$, and blue means high values. The figure shows that the partition is finer for larger values of $\mu$. In Figure \ref{regret}, we show accumulated regret of A-BLiN and zooming algorithm \cite{kleinberg2008multi}. In the figure, different background colors represent different batches of A-BLiN. For the total time horizon $T=80000$, A-BLiN needs $4$ rounds of communications.} 
        \label{f:exp}
    \end{figure*} 
    
    
    We let the time horizon $T=80000$, and report the accumulated regret in Figure \ref{regret}. The regret curve is sublinear, which agrees with the regret bound (\ref{up_bound}). Besides, different background colors in Figure \ref{regret} represent different batches. For the total time horizon $T=80000$, A-BLiN only needs $4$ rounds of communications. We also present regret curve of zooming algorithm \cite{kleinberg2008multi} for comparison. Different from zooming algorithm, regret curve of A-BLiN is approximately piecewise linear, which is because the strategy of BLiN does not change within each batch. Results of more repeated experiments, as well as experimental results of D-BLiN, are in Appendix \ref{app:blin}. Our code is available at \url{https://github.com/FengYasong-fifol/Batched-Lipschitz-Narrowing}.

\section{Conclusion}
In this paper, we study Lipschitz bandits with communication constraints, and propose the BLiN algorithm as a solution. We prove that BLiN only need $ \mathcal{O} \left( \log\log T \right) $ rounds of communications to achieve the optimal regret rate of best previous Lipschitz bandit algorithms \cite{kleinberg2008multi,bubeck2008tree} that need $T$ batches. This improvement in number of the batches significantly saves data communication costs. We also provide complexity analysis for this problem. We show that $\Omega(\log\log T)$ rounds of communications are necessary for any algorithm to optimally solve Lipschitz bandit problems. Hence, BLiN algorithm is optimal.

\appendix

\section{Proof of Corollary \ref{coro}}\label{app:coro}

\noindent\textbf{Corollary \ref{coro}.} For Lipschitz bandit problems with ambient dimension $d$, zooming dimension $d_z\leq d$ and time horizon $T$, any algorithm needs $\Omega(\log\log T)$ rounds of communications to achieve the optimal regret rate $\mathcal{O}\left(T^{\frac{d_z+1}{d_z+2}}\right)$.

\begin{proof}
    From Theorem \ref{thm:lower-adaptive-dz-intro}, the expected regret is lower bounded by 
	\begin{align*} 
	    \E \left[ R_T(\pi) \right] 
	    \geq \frac{1}{512B^2}T^{\frac{1-\frac{1}{d_z+2}}{1-\left(\frac{1}{d_z+2}\right)^B}}.
	\end{align*}
    Here we seek for the minimum $B$ such that
    \begin{align}\label{lower_b_1}
        \frac{\frac{1}{512B^2} T^{\frac{1-\frac{1}{d_z+2}}{1-\left(\frac{1}{d_z+2}\right)^B}}} {T^{\frac{d_z+1}{d_z+2}}}\leq C
    \end{align}
    for some constant $C$.
    
    Calculation shows that
    \begin{align}\label{lower_b_2}
        \frac{\frac{1}{512B^2} T^{\frac{1-\frac{1}{d_z+2}}{1-\left(\frac{1}{d_z+2}\right)^B}}} {T^{\frac{d_z+1}{d_z+2}}}=\frac{1}{512B^2}\left(T^\frac{d_z+1}{d_z+2}\right)^\frac{1}{(d_z+2)^B-1}.
    \end{align}
    
    Substituting (\ref{lower_b_2}) to (\ref{lower_b_1}) and taking log on both sides yield that
    \begin{align*}
        \frac{d_z+1}{d_z+2}\cdot\frac{\log T}{(d_z+2)^B-1}\leq\log (512B^2C)
    \end{align*}
    and
    \begin{align*}
        (d_z+2)^B\geq\frac{d_z+1}{(d_z+2)\log (512B^2C)}\cdot\log T+1.
    \end{align*}
    Taking log on both sides again yields that
    \begin{equation}\label{lower-coro}
        B\geq\frac{\log\left[\left(\frac{d_z+1}{(d_z+2)\log (512B^2C)}\right)\log T+1\right]}{\log (d_z+2)}.
    \end{equation}
    We use $B_{\min}$ to denote the minimum $B$ such that inequality (\ref{lower-coro}) holds. Calculation shows that (\ref{lower-coro}) holds for
    \begin{equation*}
        B=B_*\triangleq \frac{\log\left[\left(\frac{d_z+1}{(d_z+2)\log (512C)}\right)\log T+1\right]}{\log (d_z+2)},
    \end{equation*}
    so we have $B_{\min}\leq B_*$. Then since the RHS of (\ref{lower-coro}) decreases with $B$, we have
    \begin{align*}
        B_{\min}\geq\frac{\log\left[\left(\frac{d_z+1}{(d_z+2)\log (512B_{\min}^2C)}\right)\log T+1\right]}{\log (d_z+2)}
        \geq\frac{\log\left[\left(\frac{d_z+1}{(d_z+2)\log (512B_*^2C)}\right)\log T+1\right]}{\log (d_z+2)}.
    \end{align*}
    Therefore, $\Omega(\log\log T)$ rounds of communications are necessary for any algorithm to optimally solve Lipschitz bandit problems.
\end{proof}

\section{Space Complexity Analysis of A-BLiN}\label{app:space}

Let the A-BLiN run contains $B+1$ batches: $B$ batches in the for-loop and a clean-up batch. We use $\gamma_m$ to denote number of cubes in batch $m$, that is, $\gamma_m=|\mathcal{A}_m|$, for $1\leq m\leq B$. Then it is easy to see that the space complexity is linear in $\max_{1\leq m\leq B}\gamma_m$. In the following, we bound $\gamma_m$ for each $m$ to obtain the space complexity of A-BLiN.

Equation (\ref{eq:factor-rm}) and (\ref{eq:firstB}) yields that 
\begin{align*}
    \gamma_m\cdot\frac{16\log T}{r_m^2}\cdot8r_{m-1}\leq T^{\frac{d_z+1}{d_z+2}}\cdot128C_z(\log T)^{\frac{1}{d_z+2}},
\end{align*}
and thus
\begin{equation*}
    \gamma_m\leq C_z T^{\frac{d_z+1}{d_z+2}} (\log T)^{-\frac{d_z+1}{d_z+2}}\cdot\frac{r_m^2}{r_{m-1}}.
\end{equation*}
Since $r_m\leq1$ and $r_m\leq r_{m-1}$, we further have
\begin{align*}
    \gamma_m\leq C_z T^{\frac{d_z+1}{d_z+2}} (\log T)^{-\frac{d_z+1}{d_z+2}}.
\end{align*}
The above inequality is satisfied for each $1\leq m\leq B$. Consequently, the space complexity of A-BLiN is upper bounded by $\max_{1\leq m\leq B}\gamma_m=\mathcal{O}\left(T^{\frac{d_z+1}{d_z+2}} (\log T)^{-\frac{d_z+1}{d_z+2}}\right)$.

\section{Proof of Theorem \ref{main_t_d}} 
\label{proof:dblin}
\noindent\textbf{Theorem \ref{main_t_d}.} With probability exceeding $1-\frac{2}{T^6}$, the $T$-step total regret $R(T)$ of BLiN with Doubling Edge-length Sequence (D-BLiN) satisfies
    \begin{equation}\label{up_bound_d} 
        R(T)\leq(512C_z+16)\cdot T^\frac{d_z+1}{d_z+2} (\log T)^{\frac{1}{d_z+2}},
    \end{equation} 
    where $d_z$ is the zooming dimension of the problem instance. In addition, D-BLiN only needs no more than $\frac{\log T-\log\log T}{d_z + 2} + 2$ rounds of communications to achieve this regret rate.
    
\begin{proof}
    Since $r_m=\frac{r_{m-1}}{2}$ for Doubling Edge-length Sequence, Lemma \ref{lem:eli} implies that every cube $C\in\mathcal{A}_m$ is a subset of $S(16r_m)$. Thus from the definition of zooming number, we have 
    \begin{align} 
        | \A_m | \le N_{r_m}. \label{eq:bond-Am}
    \end{align} 
    
    
    Fix any positive number $B$. 
    Also by Lemma \ref{lem:eli}, we know that any arm played after batch $B$ incurs a regret bounded by $16 r_{B}$, since the cubes played after batch $B$ have edge length no larger than $r_{B}$. Then the total regret occurs after the first $B$ batch is bounded by $16r_BT$.
    
    Thus the regret $R(T)$ can be bounded by 
    \begin{align}
    	R(T)
    	&\le 
    	\sum_{m=1}^{B}\sum_{C\in\mathcal{A}_m}\sum_{i=1}^{n_m}\Delta_{x_{C,i}} + 16 r_{B} T, 
    	\label{eq:reg-decomp}
    \end{align} 
    where the first term bounds the regret in the first $B$ batches of D-BLiN, and the second term bounds the regret after the first $B$ batches. If the algorithm stops at batch $\widetilde{B}<B$, we define $\mathcal{A}_m=\varnothing$ for any $\widetilde{B}<m\leq B$ and inequality (\ref{eq:reg-decomp}) still holds. 
    
    By Lemma \ref{lem:eli}, we have $\Delta_{C,i} \le 16 r_m$ for all $C \in \A_m$. We can thus bound (\ref{eq:reg-decomp}) by
    \begin{align}
        R(T)
        \le&
        \sum_{m=1}^{B} |\A_m| \cdot n_m \cdot 16 r_m + 16 r_{B} T \nonumber \\
        \le& 
        \sum_{m=1}^{B} N_{r_m} \cdot n_m \cdot 16 r_m + 16 r_{B} T, \label{eq:use-bound-Am} \\
        \le& 
        \sum_{m=1}^{B} N_{r_m} \cdot \frac{16 \log T }{r_m^2} \cdot 16 r_m + 16 r_{B} T, \label{eq:use-nm}\\
        =&
        \sum_{m=1}^{B} N_{r_m} \cdot \frac{256 \log T }{r_m} + 16 r_{B} T\nonumber,
    \end{align}
    where (\ref{eq:use-bound-Am}) uses (\ref{eq:bond-Am}), and (\ref{eq:use-nm}) uses equality $ n_m = \frac{16 \log T}{ r_m } $. Since $ r_m = 2^{-m+1} $ and $ N_{r_m} \le C_z r_m^{-d_z} \le C_z\cdot2^{(m-1)d_z} $, we have
    \begin{align*} 
        R(T) 
        \le256C_z  \sum_{m=1}^{B}  \frac{ 2^{(m-1) d_z } \log T  }{ 2^{-(m-1)} } + 16 \cdot 2^{-{B}+1} T. 
    \end{align*} 
    
    This inequality holds for any positive $B$. By choosing $B^* = 1 + \frac{ \log \frac{T}{\log T} }{d_z + 2}$, we have 
    \begin{align*}
        R(T) 
        \le& 
        512C_z \cdot 2^{ \left( B^*-1 \right) \left( d_z + 1 \right)} \log T + 16 \cdot T \cdot 2^{-B^* + 1}
        \le
        (512C_z+16)\cdot T^{\frac{d_z + 1}{d_z + 2}} (\log T)^{\frac{1}{d_z+2}}.
    \end{align*}
    The above analysis implies that we can achieve the optimal regret rate $\widetilde{\mathcal{O}}\left(T^\frac{d_z+1}{d_z+2}\right)$ by letting the \emph{for-loop} run $B^*$ times and finishing the remaining rounds in the \emph{Cleanup} step. In other words, $B^*+1$ rounds of communications are sufficient for D-BLiN to achieve the regret bound (\ref{up_bound_d}).
\end{proof}

\section{Regret Upper Bound of BLiN with a Fixed Number of Batches}\label{app:fix_B}

This section studies the cases where a hard upper bound $B$ on the number of batches is imposed. In such cases, we simply apply BLiN by executing the for-loop $B-1$ times, and then run the clean-up step. Since some batches may be skipped in the for-loop, the total number of batches is upper bounded by $B$. 
The regret in such cases is described by a quantity called \emph{effective number of batches} (written $\text{Eff}(B)$). The quantity $ \text{Eff} (B) $ counts number of batches where finer partitioning of the arm space occurs. Further in Lemma \ref{lem:eff}, 
we show that $ \text{Eff}(B) = \mathcal{O} (\log \log T) $. 

Firstly, we introduce some notations. In the proof of Theorem \ref{t_race}, we bound the regret of the odd batches ($m=2k-1$ and $m<B$), the even batches ($m=2k$ and $m<B$) and the clean-up batch (the $B$-th batch) separately. For convenience, we denote
\begin{equation*}
    R_{\mathrm{odd}}=\sum_{m=2k-1,\;m<B}R_m\quad\text{and}\quad R_{\mathrm{even}}=\sum_{m=2k,\;m<B}R_m,
\end{equation*}
where $R_m$ is the regret of the $m$-th batch. Additionally, we let $B_e=\left\lfloor\frac{B-1}{2}\right\rfloor$, and thus $2B_e$ is the last even batch before the $B$-th batch. As is mentioned in the paper (the paragraph before Theorem \ref{t_race}), when using BLiN with rounded ACE Sequence, if there exists $k<m$ such that $\wtr_m\geq\wtr_{k}$, then we skip the $m$-th batch. We use $\text{Eff}(B)$ to denote the number of effective batches, that is, the batches which are not skipped.

By omitting the equality $B^* = \frac{\log\log \frac{T}{\log T}-\log(d_z+2)}{\log \frac{d+2}{d+1-d_z}}$ and using the definition of rounded ACE Sequence, the arguments in the proof of Theorem \ref{t_race} can be directly applied to the case of fixed $B$. Specifically, inequality (\ref{r:odd}) yields that 
\begin{equation*}
    R_{\text{odd}}\lesssim \text{Eff}(B)\cdot T^{\frac{d_z+1}{d_z+2}}\cdot(\log T)^{\frac{1}{d_z+2}};
\end{equation*}
inequality (\ref{eq:r-even-fixB}) yields that
\begin{align*}
    R_{\text{even}}
    \lesssim\wtr_{2B_e}^{-d_z-1}\cdot\log T
    \lesssim\left(\frac{T}{\log T}\right)^{-\frac{d_z+1}{d_z+2}\left(\frac{d+1-d_z}{d+2}\right)^{\left\lfloor\frac{B-1}{2}\right\rfloor}}\cdot\left(\frac{T}{\log T}\right)^{\frac{d_z+1}{d_z+2}}\cdot\log T
    \leq T^{\frac{d_z+1}{d_z+2}}\cdot(\log T)^{\frac{1}{d_z+2}};
\end{align*}
and inequality (\ref{r:last}) yields that 
\begin{align*}
    R_B\lesssim\wtr_{2B_e}\cdot T
    \lesssim T^{\frac{1}{d_z+2}\left(\frac{d+1-d_z}{d+2}\right)^{\left\lfloor\frac{B-1}{2}\right\rfloor}}\cdot T^{\frac{d_z+1}{d_z+2}}\cdot(\log T)^{\frac{1}{d_z+2}}.
\end{align*}
Thus, the $T$-step total regret is bounded by
\begin{align}\label{eq:fixb}
        R(T)
    =R_{\text{odd}}+R_{\text{even}}+R_B
    \lesssim\left(\text{Eff}(B)+T^{\frac{1}{d_z+2}\left(\frac{d+1-d_z}{d+2}\right)^{\left\lfloor\frac{B-1}{2}\right\rfloor}}\right)\cdot T^{\frac{d_z+1}{d_z+2}}\cdot(\log T)^{\frac{1}{d_z+2}}.
\end{align}
Furthermore, because of the existence of the rounding step, the actual number of batches of BLiN with rounded ACE Sequence has the following upper bound.

\begin{lemma}\label{lem:eff}
    Let $T$ be the total time horizon. When applying BLiN with rounded ACE Sequence $\{\widetilde{r}_m\}$ and any number of batches $B$, the effective number of batches is upper bounded by
    \begin{equation*}
        \mathrm{Eff}(B)=\mathcal{O}(\log\log T).
    \end{equation*}
\end{lemma}

\begin{proof}
    The rounded ACE sequence $\{\widetilde{r}_m\}_{m\in\mathbb{N}}$ is defined as $\widetilde{r}_m=2^{-\alpha_k}$ for $m=2k-1$ and $\widetilde{r}_m=2^{-\beta_k}$ for $m=2k$, where $\alpha_k=\lfloor c_1\cdot\frac{1-\eta^k}{1-\eta}\rfloor$ and $\beta_k=\lceil c_1\cdot\frac{1-\eta^k}{1-\eta}\rceil$. As is explained in Remark \ref{remark:race}, if there exists integer $M$ and $k$ such that $M< c_1\cdot\frac{1-\eta^k}{1-\eta}< c_1\cdot\frac{1-\eta^{k+1}}{1-\eta}<M+1$, then the rounding step yields $\widetilde{r}_{2k-1}=\widetilde{r}_{2k+1}<\widetilde{r}_{2k}=\widetilde{r}_{2k+2}$, so the $(2k+1)$-th and the $(2k+2)$-th batches are skipped. We note that the sequence $\{c_1\cdot\frac{1-\eta^k}{1-\eta}\}$ is increasing and $\lim_{k\to\infty}c_1\cdot\frac{1-\eta^k}{1-\eta}=c_1\cdot\frac{1}{1-\eta}$. It is easy to verify that if inequality
    \begin{equation}\label{ineq:eff}
        c_1\cdot\frac{1}{1-\eta}-c_1\cdot\frac{1-\eta^{k_0}}{1-\eta}<1
    \end{equation}
    is satisfied for some $k_0$, then there will be at most $2$ effective batches after the $2k_0$-th batch. As a consequence, for any $B$, the number of effective batches is upper bounded by $2k_0+2$.

    By choosing $k_0=\frac{\log\log T}{\log\frac{d+2}{d+1-d_z}}$, we have
    \begin{equation*}
        k_0>\frac{\log\left(\frac{1}{d_z+2}\log \frac{T}{\log T}\right)}{\log\frac{d+2}{d+1-d_z}}=\frac{\log\frac{1-\eta}{c_1}}{\log\eta},
    \end{equation*}
    and (\ref{ineq:eff}) is satisfied. Therefore, we conclude that $\text{Eff}(B)\leq\frac{2\log\log T}{\log\frac{d+2}{d+1-d_z}}+2$.
\end{proof}

Combining (\ref{eq:fixb}), Lemma \ref{lem:eff} and the inequality $\left\lfloor\frac{B-1}{2}\right\rfloor\geq\frac{B}{2}-1$, we conclude that the $T$-step regret of BLiN with rounded ACE Sequence and $B$ batches is upper bounded by
\begin{equation*}
    R(T)\lesssim T^{\frac{1}{d_z+2}\left(\frac{d+1-d_z}{d+2}\right)^{\frac{B}{2}-1}}\cdot T^{\frac{d_z+1}{d_z+2}}\cdot(\log T)^{\frac{1}{d_z+2}}.
\end{equation*}
This upper bound is slightly larger than our lower bound in Theorem \ref{thm:lower-adaptive-dz-intro}, and they matches when $B=\Theta(\log\log T)$.

\section{Proof of Theorem \ref{thm:lower-static-dz}}\label{app:lower-dz-static}
\noindent\textbf{Theorem \ref{thm:lower-static-dz}.} Consider Lipschitz bandit problems with time horizon $T$, ambient dimension $d$ and zooming dimension $d_z\leq d$ such that the grid of reward communication $\mathcal{T}$ is static and determined before the game. If $B$ rounds of communications are allowed, then for any policy $\pi$, there exists a problem instance with zooming dimension $d_z$ such that
    \begin{equation*}
        \E[R_T(\pi)]\geq \frac{1}{64e^\frac{1}{16}}\cdot T^{\frac{1-\frac{1}{d_z+2}}{1-\left(\frac{1}{d_z+2}\right)^B}}.
    \end{equation*}

\begin{proof}
    Fixing $k$ and $1\leq i\leq M_k$, we show that $\mu_{k,i}$ is $1$-Lipschitz and the zooming dimension equals to $d_z$.
    
    Firstly, for any $(\alpha_1,\beta_1)$ and $(\alpha_2,\beta_2)$, we have
    \begin{align*}
        \mu_{k,i}((\alpha_1,\beta_1))-\mu_{k,i}((\alpha_2,\beta_2))
        \leq&\mu_{k,i}((\alpha_1,\beta_1))-\mu_{k,i}((\alpha_1,\beta_2))
        +\mu_{k,i}((\alpha_1,\beta_2))-\mu_{k,i}((\alpha_2,\beta_2))\\
        \leq&\frac{1}{2}\|\beta_1-\beta_2\|_\infty+\frac{1}{2}\big(\nu_{k,i}((\alpha_1,0_{d-d_z}))-\nu_{k,i}((\alpha_2,0_{d-d_z}))\big)\\
        \leq&\frac{1}{2}(\|\beta_1-\beta_2\|_\infty+\|\alpha_1-\alpha_2\|_\infty)\\
        \leq&\|(\alpha_1-\alpha_2,\beta_1-\beta_2)\|_\infty.
    \end{align*}
    Therefore, $\mu_{k,i}$ is $1$-Lipschitz.

    Secondly, for any $r\geq r_k$, (\ref{eq:constuct3-dz}) yields that $S(16r)=[0,1]^{d_z}\times[0, 32r]^{d-d_z}$. Therefore, we have $N_r=32^{d-d_z}r^{-d_z}$, and the zooming dimension equals to $d_z$.

    Then we show that an argument similar to Theorem \ref{st_lb} yields the lower bound we need. For any $k>1$, we prove that no algorithm can distinguish instances in $\mathcal{I}_k$ from one another in the first $(k-1)$ batches, so the worst-case regret is at least $r_kt_k$, which equals to $\frac{t_k}{t_{k-1}^{\frac{1}{d_z+2}}}$. For the first batch $(0, t_1]$, we can easily construct a set of instances where the worst-case regret is at least $t_1$, since no information is available during this time. Thus, there exists a problem instance such that
    \begin{align*}
        \E[R_T(\pi)]\gtrsim\max\left\{t_1,\frac{t_2}{t_1^\frac{1}{d_z+2}},\cdots,\frac{t_B}{t_{B-1}^\frac{1}{d_z+2}}\right\} . 
    \end{align*}
    Since $0<t_1<\cdots<t_B=T$, the inequality in Theorem \ref{thm:lower-static-dz} follows.
    
    Recall that each $u_{k,i}$ is in $\A_{d_z}$. For convenience, we write $u_{k,i}=(\hat{u}_{k,i},0_{d-d_z})$. Let $H_{k,i}=\B_{d_z}(\hat{u}_{k,i},\frac{3}{8}r_k)\times[0,1]^{d-d_z}$, where $\B_{d_z}(\hat{u}_{k,i},\frac{3}{8}r_k)$ denotes the $d_z$-dimensional ball with center $\hat{u}_{k,i}$ and radius $\frac{3}{8}r_k$. It is easy to verify the following properties of construction (\ref{eq:constuct1-dz}),(\ref{eq:constuct2-dz}) and (\ref{eq:constuct3-dz}):
    \begin{enumerate}
        \item For any $2\leq i\leq M_k$, $\mu_{k,i}(z)=\mu_{k,1}(z)$ for any $z\in\mathcal{A}\setminus H_{k,i}$;
        \item For any $2\leq i\leq M_k$, $\mu_{k,1}(z)\leq\mu_{k,i}(z)\leq\mu_{k,1}(z)+\frac{r_k}{4}$, for any $z\in H_{k,i}$;
        \item For any $1\leq i\leq M_k$, under $I_{k,i}$, pulling an arm that is not in $H_{k,i}$ incurs a regret at least $\frac{r_k}{16}$. 
    \end{enumerate}

    The lower bound of expected regret relies on the following lemma.
    \begin{lemma}\label{lem:kld-dz}
    	For any policy $\pi$, there exists a problem instance $I\in\mathcal{I}_k$ such that
    	\begin{equation*}
    	    \E \[ R_T(\pi) \] \geq  \frac{r_k}{64}\cdot\sum_{j=1}^B \( t_j-t_{j-1} \) \exp\left\{-\frac{t_{j-1}r_k^2}{32(M_k-1)}\right\}.
    	\end{equation*}
    	 
    \end{lemma} 

    \begin{proof}
        
        Let $ x_t $ denote the choices of policy $\pi$ at time $t$, and $y_t$ denote the reward. Additionally, for $t_{j-1}<t\leq t_j$, we define $\Pr_{k.i}^t$ as the distribution of sequence $\left(x_1,y_1,\cdots,x_{t_{j-1}},y_{t_{j-1}}\right)$ under instance $I_{k,i}$ and policy $\pi$. It holds that 
        \begin{align}\label{eq:spl-dz} 
            \sup_{I\in\mathcal{I}_k} \E R_T(\pi)  
            \geq
            \frac{1}{M_k}\sum_{i=1}^{M_k} \E_{\Pr_{k,i}} \left[ R_T(\pi) \right]
            \geq
            \frac{1}{M_k}\sum_{i=1}^{M_k} \sum_{t=1}^T\E_{\Pr_{k,i}^t} \left[ R^t(\pi) \right] 
            \geq
            \frac{r_k}{16}\sum_{t=1}^T\frac{1}{M_k}\sum_{i=1}^{M_k} \Pr_{k,i}^t(x_t\notin H_{k,i}),
        \end{align} 
        where $R^t(\pi)$ denotes the regret incurred by policy $\pi$ at time $t$. 
            
        From our construction, it is easy to see that $H_{k,i}\cap H_{k,j}=\varnothing$ for any $i\neq j$, so we can construct a test $\Psi$ such that $x_t\in H_{k,i}$ implies $\Psi= i$.  
        Then from Lemma \ref{lem:test-tree}, 
        \begin{align*}
            \frac{1}{M_k}\sum_{i=1}^{M_k} \Pr_{k,i}^t(x_t\notin H_{k,i})  
            \geq \;  
            \frac{1}{M_k}\sum_{i=1}^{M_k} \Pr_{k,i}^t(\Psi\neq i) 
            \geq \;  
            \frac{1}{2M_k}\sum_{i=2}^{M_k} \exp 
            \left\{-D_{KL}\left(\Pr_{k,1}^t\|\Pr_{k,i}^t\right)\right\}.
        \end{align*} 

        To avoid notational clutter, for any $s,s'$ ($s \ge s'$), define 
        \begin{align*}
            (\x, \y )_{:s}^{:s'} =& \; (x_1, y_1,\cdots,x_{s'}, y_{s'}, \cdots ,x_s) . 
        \end{align*}
        
        Now we calculate $D_{KL}\left(\Pr_{k,1}^t\|\Pr_{k,i}^t\right)$. From the chain rule of KL-Divergence, we have
        \begin{align}
            & \; D_{KL}\left(\Pr_{k,1}^t\|\Pr_{k,i}^t\right) \nonumber \\
            =& \;
            D_{KL}\left(\Pr_{k,1}^t \( (\x ,\y)_{:t_{j-1}}^{:t_{j-1}} \) \| \Pr_{k,i}^t \( (\x ,\y)_{:t_{j-1}}^{:t_{j-1}} \) \right)\nonumber\\
            =& \;
            D_{KL}\left(\Pr_{k,1}^t \(  (\x ,\y)_{:t_{j-1}}^{:t_{j-1} - 1} \)  \| \Pr_{k,i}^t \( (\x ,\y)_{:t_{j-1}}^{:t_{j-1} - 1} \) \right) 
            +\E_{\Pr_{k,1}}\left(D_{KL}\left(\Pr_{k,1}^t(y_{t_{j-1}}|x_{t_{j-1}}) \| \Pr_{k,i}^t(y_{t_{j-1}}|x_{t_{j-1}})\right)\right) \label{chain_eq1-dz}\\
            =& \;
            D_{KL}\left(\Pr_{k,1}^t \(  (\x ,\y)_{:t_{j-1}-1}^{:t_{j-1} - 1} \)  \| \Pr_{k,i}^t \( (\x ,\y)_{:t_{j-1}-1}^{:t_{j-1} - 1} \) \right) 
            +\E_{\Pr_{k,1}}\left(D_{KL}\left(\Pr_{k,1}^t(y_{t_{j-1}}|x_{t_{j-1}}) \| \Pr_{k,i}^t(y_{t_{j-1}}|x_{t_{j-1}})\right)\right) \label{chain_eq1.5-dz}\\
            \leq& \;
            D_{KL}\left(\Pr_{k,1}^t\( (\x ,\y)_{:t_{j-1}-1}^{:t_{j-1} - 1} \) \| \Pr_{k,i}^t\( (\x ,\y)_{:t_{j-1}-1}^{:t_{j-1} - 1} \)\right)
            +\E_{\Pr_{k,1}}\left(D_{KL}\left(N(\mu_{k,1}(x_{t_{j-1}}),1) \| N(\mu_{k,i}(x_{t_{j-1}}),1)\right)\right) \label{chain_eq2-dz}\\
            =& \;
            D_{KL}\left(\Pr_{k,1}^t \( (\x ,\y)_{:t_{j-1}-1}^{:t_{j-1} - 1} \) \| \Pr_{k,i}^t \( (\x ,\y)_{:t_{j-1}-1}^{:t_{j-1} - 1} \)\right) 
            +\E_{\Pr_{k,1}}\left(\frac{1}{2}\left(\mu_{k,1}(x_{t_{j-1}})-\mu_{k,i}(x_{t_{j-1}})\right)^2\right) \nonumber\\
            \leq& \;
            D_{KL}\left(\Pr_{k,1}^t \( (\x ,\y)_{:t_{j-1}-1}^{:t_{j-1} - 1} \) \| \Pr_{k,i}^t \( (\x ,\y)_{:t_{j-1}-1}^{:t_{j-1} - 1} \)\right)
            +\E_{\Pr_{k,1}}\left(\bm{1}_{\{x_{t_{j-1}}\in S_{k,i}\}}\cdot\frac{1}{2}\left(\frac{r_k}{4}\right)^2\right)\label{chain_use_prop-dz}\\
            =& \;
            D_{KL}\left(\Pr_{k,1}^t \( (\x ,\y)_{:t_{j-1}-1}^{:t_{j-1} - 1} \) \| \Pr_{k,i}^t \( (\x ,\y)_{:t_{j-1}-1}^{:t_{j-1} - 1} \)\right)
            +\frac{r_k^2}{32}\cdot\Pr_{k,1}\left(x_{t_{j-1}}\in S_{k,i}\right),\label{chain_singal-dz}
        \end{align}
        where (\ref{chain_eq1-dz}) uses chain rule for KL-divergence and the conditional independence of the reward, (\ref{chain_eq1.5-dz}) removes dependence on $x_{t_{j - 1}}$ in the first term by another use of chain rule and the fact that the distribution of $x_{t_{j-1}}$ is fully determined by the policy and the distribution of $(\x ,\y)_{:t_{j-1}-1}^{:t_{j-1} - 1}$, (\ref{chain_eq2-dz}) uses that the rewards are corrupted by a standard normal noise, and (\ref{chain_use_prop-dz}) uses the first two properties of the construction.
        
        Since (\ref{chain_singal-dz}) holds for all $t \le t_{j-1}$, we conclude that
        \begin{align}\label{chain-dz}
            D_{KL}\left(\Pr_{k,1}^t\|\Pr_{k,i}^t\right)\leq\frac{r_k^2}{32}\sum_{s\leq t_{j-1}}\Pr_{k,1}\left(x_s\in H_{k,i}\right)=\frac{r_k^2}{32} \E_{\Pr_{k,1}}\tau_i,
        \end{align}
        where $\tau_i$ denotes the number of pulls of arms in $H_{k,i}$ before the batch containing $t$. Then for all $t \in ( t_{j-1}, t_j]$, we have  
        \begin{align}
            \frac{1}{M_k}\sum_{i=1}^{M_k} \Pr_{k,i}^t(x_t\notin H_{k,i})
            \geq& \; 
            \frac{1}{2M_k}\sum_{i=2}^{M_k} \exp\left\{-\frac{r_k^2}{32} \E_{\Pr_{k,1}}\tau_i\right\} \nonumber \\ 
            \geq& \;  
            \frac{M_k-1}{2M_k}\exp\left\{-\frac{r_k^2}{32(M_k-1)}\sum_{i=2}^{M_k}\E_{\Pr_{k,1}} \tau_i \right\} \label{eq:use-jensen-dz}  \\
            \geq& \; 
            \frac{1}{4}\exp\left\{-\frac{r_k^2t_{j-1}}{32(M_k-1)}\right\}, \label{eq:hyp-dz}
        \end{align} 
        where (\ref{eq:use-jensen-dz}) uses the Jensen' inequality, and (\ref{eq:hyp-dz}) uses the fact that $\sum_{i=2}^{M_k}\tau_i\leq t_{j-1}$. Finally, we substitute (\ref{eq:hyp-dz}) to (\ref{eq:spl-dz}) to finish the proof of Lemma \ref{lem:kld-dz}. 
    \end{proof} 
    
	Since $M_k=t_{k-1}r_k^2$, the expected regret of policy $\pi$ satisfies
	\begin{align*} 
		\E \[ R_T(\pi) \] 
		&\geq 
		\frac{r_k}{64} \cdot\sum_{j=1}^B \( t_j-t_{j-1} \) \exp\left\{-\frac{t_{j-1}r_k^2}{32(M_k-1)}\right\}\\ 
		&\geq
		\frac{r_k}{64} \cdot\sum_{j=1}^B \( t_j-t_{j-1} \) \exp\left\{-\frac{t_{j-1}r_k^2}{16M_k}\right\}\\
		&\geq 
		\frac{r_k}{64} \cdot\sum_{j=1}^B \( t_j-t_{j-1}\) \exp\left\{-\frac{t_{j-1}}{16t_{k-1}}\right\}
	\end{align*}
	on an instance $I_{k,i}\in\I_k$.
	
	By omitting terms with $j > k$ in the above summation, we have 
	\begin{align*}
	    \E[R_T(\pi)]&\geq\frac{r_k}{64} \cdot \sum_{j=1}^B \( t_j-t_{j-1}\) \exp\left\{-\frac{t_{j-1}}{16t_{k-1}}\right\}  \\
	    &\ge 
	    \frac{r_k}{64} \cdot \sum_{j=1}^{k} \( t_j-t_{j-1}\) \exp\left\{-\frac{1}{16}\right\} \\
	    &= \frac{1}{64e^\frac{1}{16}}r_kt_k
		= \frac{1}{64e^\frac{1}{16}}\cdot\frac{t_k}{t_{k-1}^\frac{1}{d_z+2}}.
	\end{align*}
   	The above analysis can be applied for any $k>1$. For the first batch $(0, t_1]$, we can easily construct a set of instances where the worst-case regret is at least $t_1$, since no information is available during this time.  Thus, there exists a problem instance such that
   	\begin{align*}
   	    \E[R_T(\pi)]\geq\frac{1}{64e^\frac{1}{16}}\max\left\{t_1,\frac{t_2}{t_1^\frac{1}{d_z+2}},\cdots,\frac{t_B}{t_{B-1}^\frac{1}{d_z+2}}\right\}.
   	\end{align*}
    
    Since $0<t_1<\cdots<t_B=T$, we further have
    \begin{align*}
        \max\left\{t_1,\frac{t_2}{t_1^\varepsilon},\cdots,\frac{t_B}{t_{B-1}^\varepsilon}\right\}
        \geq\left(t_1^{\varepsilon^{B-1}}\cdot\left(\frac{t_2}{t_1^\varepsilon}\right)^{\varepsilon^{B-2}}\cdots\left(\frac{t_{B-1}}{t_{B-2}^\varepsilon}\right)^{\varepsilon}\cdot\frac{t_B}{t_{B-1}^\varepsilon}\right)^{\frac{1}{\sum_{i=1}^{B-1}\varepsilon^i}}
        =T^{\frac{1-\varepsilon}{1-\varepsilon^B}},
    \end{align*}
    where $\varepsilon=\frac{1}{d_z+2}$. Combining the above two inequalities, we conclude that
    
   	\begin{equation*}
	    \E \[ R_T(\pi) \]
	    \geq
	    \frac{1}{64e^\frac{1}{16}}\cdot T^{\frac{1-\frac{1}{d_z+2}}{1-\left(\frac{1}{d_z+2}\right)^B}}. \qedhere
	\end{equation*}
    
\end{proof} 

\section{Proof of Theorem \ref{thm:lower-adaptive-dz}}\label{app:lower-dz-adaptive}

\noindent\textbf{Theorem \ref{thm:lower-adaptive-dz}.} Consider Lipschitz bandit problems with time horizon $T$, ambient dimension $d$ and zooming dimension $d_z\leq d$ such that the grid of reward communication $\mathcal{T}$ is adaptively determined by the player. If $B$ rounds of communications are allowed, then for any policy $\pi$, there exists a problem instance with zooming dimension $d_z$ such that 
    	\begin{align*} 
    	    \E \left[ R_T(\pi) \right] 
    	    \geq \frac{1}{512B^2}T^{\frac{1-\frac{1}{d_z+2}}{1-\left(\frac{1}{d_z+2}\right)^B}}.
    	\end{align*}

\begin{proof}
    
    The main argument in the proof of Theorem \ref{thm:lower-adaptive-dz} is similar to that of Theorem \ref{thm:ada_lb}. To construct hard instances in the $d$-dimensional space, we use the `linear-decaying extension' technique, which is the same as the proof of Theorem \ref{thm:lower-static-dz}.

    To prove Theorem \ref{thm:lower-adaptive-dz}, we consider a reference static grid $\mathcal{T}_r=\{T_0,T_1,\cdots,T_B\}$, where $T_j=T^\frac{1-\varepsilon^j}{1-\varepsilon^B}$ for $\varepsilon=\frac{1}{d_z+2}$. Then we construct a series of `worlds', denoted by $\mathcal{I}_1,\cdots,\mathcal{I}_B$. Each world is a set of problem instances, and each problem instance in world $\mathcal{I}_j$ is defined by peak location set $\mathcal{U}_j$ and basic height $r_j$, where the sets $\mathcal{U}_j$ and quantities $r_j$ for $1\leq j\leq B$ are presented in the proof below. Based on these constructions, we first prove that for any adaptive grid and policy, there exists an index $j$ such that the event $A_j=\{t_{j-1}<T_{j-1},\;t_j\geq T_j\}$ happens with sufficiently high probability in world $\mathcal{I}_j$. Then similar to Theorem \ref{thm:lower-static-dz}, we prove that in world $\mathcal{I}_j$ there exists a set of problem instances that is difficult to differentiate in the first $j-1$ batches. In addition, event $A_j$ implies that $t_j\geq T_j$, so the worst-case regret is at least $r_jT_j$, which gives the lower bound we need.
    
    Firstly, we define $r_j=\frac{1}{T_{j-1}^\varepsilon B}$ and $M_j=\frac{1}{r_j^{d_z}}$, where $\varepsilon=\frac{1}{d_z+2}$.
    From the definition, we have
    \begin{align} 
        T_{j-1}r_j^2=\frac{1}{r_j^{d_z} B^{d_z+2}}\leq\frac{1}{r_j^{d_z}B^2}=\frac{M_j}{B^2}.  \label{eq:rm-relation-dz}
    \end{align}
    For $1\leq j\leq B$, we can find sets of arms $\mathcal{U}_j=\{u_{j,1},\cdots,u_{j,M_j}\}\in\mathcal{A}_{d_z}$ such that (a) $d_{\mathcal{A}}(u_{j,m},u_{j,n})\geq r_j$ for any $m\neq n$, and (b) $u_{1,M_1}=\cdots=u_{B,M_B}$. 
    
    Then we present the construction of worlds $\mathcal{I}_1,\cdots,\mathcal{I}_B$. For $1\leq j\leq B-1$, we let $\I_{j}=\{\I_{j,k}\}_{k=1}^{M_j-1}$. We first construct a set of expected reward functions $\{\nu_{j,1},\cdots,\nu_{j, M_j-1}\}$ on $\mathcal{A}_{d_z}$. For each $1\leq k\leq M_j-1$, we define $\nu_{j,k}$ as
    \begin{align}
    	\nu_{j,k}(z)=
    	\begin{cases}
    		\frac{r_1}{2}+\frac{r_j}{16}+\frac{r_B}{16},\;\text{if}\;z=u_{j,k},\\
    		\frac{r_1}{2}+\frac{r_B}{16},\;\text{if}\;z=u_{j,M_j},\\
                \max\left\{\frac{r_1}{2},\max_{u\in\mathcal{U}_j}\left\{\nu_{j,k}(u)-d_\mathcal{A}(z,u)\right\}\right\},\\
                \quad\text{if}\;z\in\A_{d_z}\setminus\mathcal{U}_j.
    	\end{cases} \label{eq:instance-ada1-dz-nu}
    \end{align}
    Based on $\{\nu_{j,k}\}_{k=1}^{M_j-1}$, the expected reward of $I_{j,k}$ is defined as
    \begin{align}\label{eq:instance-ada1-dz-mu}
    	\mu_{j,k}((\alpha,\beta))=\frac{1}{2}\left(\nu_{j,k}((\alpha,0_{d-d_z}))-\|\beta\|_\infty\right),
    \end{align}
    where $(\alpha,\beta)$ is the concatenation of $\alpha\in[0,1]^{d_z}$ and $\beta\in[0,1]^{d-d_z}$. For $j=B$, we let $\I_B=\{I_{B}\}$. We first define a function $\nu_B$ on $\A_{d_z}$ as
    \begin{align}
    	\nu_{B}(z)=
    	\begin{cases}
    		\frac{r_1}{2}+\frac{r_B}{16},\;\text{if}\;z=u_{B, M_B},\\
                \max\left\{\frac{r_1}{2},\max_{u\in\mathcal{U}_j}\left\{\mu_{j,k}(u)-d_\mathcal{A}(z,u)\right\}\right\},\\
                \quad\text{if}\;z\in\A_{d_z}/\{u_{B,M_B}\}.
    	\end{cases} \label{eq:instance-ada2-dz-nu}
    \end{align}
    Then the expected reward of $I_B$ is defined as
    \begin{align}\label{eq:instance-ada2-dz-mu}
    	\mu_{B}((\alpha,\beta))=\frac{1}{2}\left(\nu_{B}((\alpha,0_{d-d_z}))-\|\beta\|_\infty\right).
    \end{align}
    Roughly speaking, our constructions satisfy two properties: for each $j\neq B$ and $1\leq k\leq M_j-1$, 
    \begin{enumerate}
        \item $\mu_{j,k}$ is close to $\mu_B$;
        \item under $I_{j,k}$, pulling an arm that is far from $u_{j,k}$ incurs a regret at least $\frac{r_j}{16}$.
    \end{enumerate}
    The formal version of these two properties are presented in the proof of Lemma \ref{exist_pj-dz} and Lemma \ref{adagrid_lemma-dz}.
    
    Based on these constructions, we first show that for any adaptive grid $ \mathcal{T} = \{t_0,\cdots,t_B\} $, there exists an index $j$ such that $(t_{j-1},t_j]$ is sufficiently large in world $\mathcal{I}_j$. More formally, for each $j\in[B]$, and event $A_j=\{t_{j-1}<T_{j-1},\;t_j\geq T_j\}$, we define the quantities $p_j := \frac{1}{M_j-1}\sum_{k=1}^{M_j-1}\Pr_{j,k}(A_j)$ for $j\leq B-1$ and $p_B := \Pr_B(A_B)$,
    where $\Pr_{j,k}(A_j)$ denotes the probability of the event $A_j$ under instance $I_{j,k}$ and policy $\pi$. For these quantities, we have the following lemma.

\begin{lemma}
    \label{exist_pj-dz} 
    For any adaptive grid $\mathcal{T}$ and policy $\pi$, it holds that $\sum_{j=1}^B p_j\geq\frac{7}{8}.$  
\end{lemma} 

\begin{proof}
    For $1\leq j\leq B-1$ and $1\leq k\leq M_j-1$, we write $u_{j,k}=(\hat{u}_{j,k},0_{d-d_z})$. We define $H_{j,k}=\B_{d_z}(\hat{u}_{j,k},\frac{3}{8}r_j)\times[0,1]^{d-d_z}$, where $\B_{d_z}(\hat{u}_{j,k},\frac{3}{8}r_j)$ denotes the $d_z$-dimensional ball with center $\hat{u}_{j,k} $ and radius $ \frac{3}{8}r_j $. It is easy to verify the following properties of our construction (\ref{eq:instance-ada1-dz-nu}), (\ref{eq:instance-ada1-dz-mu}), (\ref{eq:instance-ada2-dz-nu}) and (\ref{eq:instance-ada2-dz-mu}):
    \begin{enumerate}
    	\item $\mu_{j,k}(z)=\mu_B (z)$ for any $z\notin H_{j,k}$;
    	\item $\mu_B (z) \leq \mu_{j,k}(z)\leq\mu_B (z)+\frac{r_j}{8}$, for any $z\in H_{j,k}$.
    \end{enumerate}
    Let $ x_t $ denote the choices of policy $\pi$ at time $t$, and $y_t$ denote the reward. For $t_{j-1}<t\leq t_j$, we define $\Pr_{j,k}^t$ (resp. $\Pr_{B}^{t}$) as the distribution of sequence $\left(x_1,y_1,\cdots,x_{t_{j-1}},y_{t_{j-1}}\right)$ under instance $I_{j,k}$ (resp. $I_B$) and policy $\pi$. Since event $A_j$ can be completely described by the observations up to time $T_{j-1}$ ($A_j$ is an event in the $\sigma$-algebra where $\Pr_{j,k}^{T_{j-1}}$ and $\Pr_B^{T_{j-1}}$ are defined on), we can use the definition of total variation to get 
    \begin{align*}
    	|\Pr_B (A_j)-\Pr_{j,k}(A_j)| 
            = \;  
            | \Pr_B^{T_{j-1}} (A_j) - \Pr_{j,k}^{T_{j-1}} (A_j) | 
    	\leq\;  
            TV \(\Pr_B^{T_{j-1}},\Pr_{j,k}^{T_{j-1}} \) . 
    \end{align*}

    For the total variation, we apply Lemma \ref{lem:pinsker-type} to get 
    \begin{align*} 
    	 \; \frac{1}{M_j-1} \sum_{k=1}^{M_j-1} TV \( \Pr_B^{T_{j-1}},\Pr_{j,k}^{T_{j-1}}\) 
    	\leq \; 
    	\frac{1}{M_j-1} \sum_{k=1}^{M_j-1} \sqrt{1-\exp\(-D_{KL}\(\Pr_B^{T_{j-1}}\|\Pr_{j,k}^{T_{j-1}}\)\)}. 
    \end{align*} 
    An argument similar to (\ref{chain-dz}) yields that
    \begin{align*}
    	D_{KL}\(\Pr_B^{T_{j-1}}\|\Pr_{j,k}^{T_{j-1}} \) \le \frac{r_j^2}{128} \E_{\Pr_B}\tau_k, 
    \end{align*}
    where $\tau_k$ denotes the number of pulls which is in $S_{j,k}$ before the batch containing $T_{j-1}$. Combining the above two inequalities gives
    \begin{align} 
    	\frac{1}{M_j-1} \sum_{k=1}^{M_j-1} TV \( \Pr_B^{T_{j-1}},\Pr_{j,k}^{T_{j-1}}\)
    	\leq& \;  
    	\frac{1}{M_j-1} \sum_{k=1}^{M_j-1}\sqrt{1- \exp\(-\frac{r_j^2}{128} \E_{\Pr_B}\tau_k\)} \nonumber \\ 
    	\leq& \;  
    	\sqrt{1-\exp\(-\frac{r_j^2}{128(M_j-1)}\E_{\Pr_B}\[\sum_{k=1}^{M_j-1}\tau_k\]\)} \label{l7_jenson-dz} \\ 
    	\leq& \;  
    	\sqrt{  1-\exp\left(-\frac{r_j^2T_{j-1}}{128(M_j-1)}\right)} \label{l7_total-dz} \\ 
    	\leq& \;  
    	\sqrt{ 1-\exp\left(-\frac{1}{64B^2}\right)} \label{l7_useeq-dz} \\ 
    	\leq& \;  \frac{1}{8B}, \nonumber 
    \end{align}
    where (\ref{l7_jenson-dz}) uses Jensen's inequality, (\ref{l7_total-dz}) uses the fact that $\sum_{k=1}^{M_j-1} \tau_k\leq T_{j-1}$, and (\ref{l7_useeq-dz}) uses (\ref{eq:rm-relation-dz}).

    Plugging the above results implies that
    \begin{align*} 
        |\Pr_B(A_j) - p_j| \le \frac{1}{M_j -1 } \sum_{k=1}^{M_j-1} |\Pr_B (A_j)-\Pr_{j,k}(A_j)| \le  \frac{1}{8B}. 
    \end{align*}
    Since $\sum_{j=1}^B \Pr_B \(  A_j \) \ge \Pr_B \( \cup_{j=1}^B A_j \) = 1 $, it holds that 
    \begin{equation*} 
        \sum_{j=1}^B p_j \geq \Pr_B (A_B) +\sum_{j=1}^{B-1} \(\Pr_B(A_j)-\frac{1}{8B} \) \geq \frac{7}{8}.\qedhere
    \end{equation*} 
\end{proof}

Lemma \ref{exist_pj-dz} implies that there exists some $j$ such that $p_j>\frac{7}{8B}$. Then we show that the worst-case regret in world $\mathcal{I}_j$ gives the lower bound we need. 

\begin{lemma}\label{adagrid_lemma-dz}
    For adaptive grid $\mathcal{T}$ and policy $\pi$, if index $j$ satisfies $p_j\geq\frac{7}{8B}$, then 
    there exists a problem instance $I$ with zooming dimension $d_z$ such that
    \begin{align*} 
        \E \left[ R_T(\pi) \right] 
        \geq \frac{1}{512B^2}T^{\frac{1-\frac{1}{d_z+2}}{1-\left(\frac{1}{d_z+2}\right)^B}}.
    \end{align*}
\end{lemma}

\begin{proof}
   Here we proceed with the case where $j \le B-1$. The case for $j = B$ can be proved analogously. 
   
   For any $1\leq k \leq M_j-1$, we construct a set of problem instances $\mathcal{I}_{j,k} = \left(I_{j,k,l} \right)_{1\leq l\leq M_j}$.
    For $l\neq k$, we first define a function $\nu_{j,k,l}$ on $\A_{d_z}$ as
    \begin{align*}
    	\nu_{j,k,l}(z)
            =
    	    \begin{cases}
    		\nu_{j,k}(z)+\frac{3r_j}{16},\;\text{if}\;z=u_{j,l},\\
    		\nu_{j,k}(z),\;\text{if}\;z\in \mathcal{U}_j\;\text{and}\;  z \neq u_{j,l} , \\ 
    		\max\left\{\frac{r_1}{2},\max_{u\in\mathcal{U}_j}\left\{\nu_{j,k,l}(u)-d_{\mathcal{A}}(z,u)\right\}\right\},\\
            \quad\text{if}\;z\in\A_{d_z}\setminus\mathcal{U}_j,
    	\end{cases}
    \end{align*} 
    where $\nu_{j,k}$ is defined in (\ref{eq:instance-ada1-dz-nu}). Then the expected reward of $I_{j,k,l}$ is defined as 
    \begin{align}\label{eq:instance-ada-jkl-dz-mu}
        \mu_{j,k,l}((\alpha,\beta))=\frac{1}{2}\left(\nu_{j,k,l}((\alpha,0_{d-d_z}))-\|\beta\|_\infty\right),
    \end{align}
    where $(\alpha,\beta)$ is the concatenation of $\alpha\in[0,1]^{d_z}$ and $\beta\in[0,1]^{d-d_z}$. For $l=k$, we let $\mu_{j,k,k}=\mu_{j,k}$. 

    We first show that each $\mu_{j,k,l}$ is $1$-Lipschitz and the zooming dimension equals to $d_z$.
    
    Firstly, for any $(\alpha_1,\beta_1)$ and $(\alpha_2,\beta_2)$, we have
    \begin{align*}
        \mu_{j,k,l}((\alpha_1,\beta_1))-\mu_{j,k,l}((\alpha_2,\beta_2))
        \leq&\mu_{j,k,l}((\alpha_1,\beta_1))-\mu_{j,k,l}((\alpha_1,\beta_2))
        +\mu_{j,k,l}((\alpha_1,\beta_2))-\mu_{j,k,l}((\alpha_2,\beta_2))\\
        \leq&\frac{1}{2}\|\beta_1-\beta_2\|_\infty
        +\frac{1}{2}\big(\nu_{j,k,l}((\alpha_1,0_{d-d_z}))-\nu_{j,k,l}((\alpha_2,0_{d-d_z}))\big)\\
        \leq&\frac{1}{2}(\|\beta_1-\beta_2\|_\infty+\|\alpha_1-\alpha_2\|_\infty)\\
        \leq&\|(\alpha_1-\alpha_2,\beta_1-\beta_2)\|_\infty.
    \end{align*}
    Therefore, $\mu_{j,k,l}$ is $1$-Lipschitz.

    Secondly, for any $r\geq r_j$, (\ref{eq:instance-ada1-dz-mu}) and (\ref{eq:instance-ada-jkl-dz-mu}) yield that $S(16r)\subset[0,1]^{d_z}\times[0, 32r]^{d-d_z}$ and $S(16r)\supset[0,1]^{d_z}\times[0, 30r]^{d-d_z}$. Therefore, we have $30^{d-d_z}r^{-d_z}\leq N_r\leq32^{d-d_z}r^{-d_z}$, and the zooming dimension equals to $d_z$.

    Then we show that an argument similar to Lemma \ref{adagrid_lemma} yields the lower bound we need. For each $1\leq k\leq M_j$, we write $u_{j,k}=(\hat{u}_{j,k},0_{d-d_z})$. We define $C_{j,k}= \B_{d_z} \(\hat{u}_{j,k},\frac{r_j}{4} \)\times[0,1]^{d-d_z}$, and our construction $\mathcal{I}_{j,k}$ has the following properties:
    \begin{enumerate}
        \item For any $l\neq k$, $\mu_{j,k,l}(z)=\mu_{j,k,k}(z)$ for any $z\notin C_{j,l}$;
        \item For any $l\neq k$, $\mu_{j,k,k}(z)\leq\mu_{j,k,l}(z)\leq\mu_{j,k,k}(z)+\frac{3r_j}{16}$ for any $z\in C_{j,l}$;
        \item For any $1\leq l\leq M_j$, under $I_{j,k,l}$, pulling an arm that is not in $C_{j,l}$ incurs a regret at least $\frac{r_j}{32}$.
    \end{enumerate}
    Let $x_t$ denote the choices of policy $\pi$ at time $t$, and $y_t$ denote the reward. For $t_{j-1}<t\leq t_j$, we define $\Pr_{j,k,l}^t$ as the distribution of sequence $\left(x_1,y_1,\cdots,x_{t_{j-1}},y_{t_{j-1}}\right)$ under instance $I_{j,k,l}$ and policy $\pi$. From similar argument in (\ref{eq:spl-dz}), it holds that
    \begin{align}
        \sup_{I\in\mathcal{I}_{j,k}}\E \[ R_T(\pi) \] 
        \geq
        \frac{r_j}{32} \sum_{t=1}^T \frac{1}{M_j} \sum_{l=1}^{M_j} \Pr_{j,k,l}^t (x_t\notin C_{j,l}). \label{eq:lower-bound-ada1-dz} 
    \end{align}
    From our construction, it is easy to see that $C_{j,k_1}\cap C_{j,k_2}=\varnothing$ for any $k_1\neq k_2$, so we can construct a test $\Psi$ such that $x_t\in C_{j,k}$ implies $\Psi= k$ . By Lemma \ref{lem:test-tree} with a star graph on $[K]$ with center $k$, we have 
    \begin{align}
        \frac{1}{M_j} \sum_{l=1}^{M_j} \Pr_{j,k,l}^t(x_t\notin C_{j,l}) 
        \ge 
        \frac{1}{ M_j } \sum_{ l \neq k } \int \min \left\{d \Pr_{j,k,k}^t,d \Pr_{j,k,l}^t\right\} . \label{eq:lower-bound-ada2-dz}
    \end{align} 
    
    Combining (\ref{eq:lower-bound-ada1-dz}) and (\ref{eq:lower-bound-ada2-dz}) gives
    \begin{align}
        \sup_{I\in\mathcal{I}_{j,k}} \E \[ R_T(\pi) \] 
        \geq& \;  
        \frac{r_j}{32} \sum_{t=1}^T \frac{1}{M_j}\sum_{l\neq k}\int \min \left\{d \Pr_{j,k,k}^t,d \Pr_{j,k,l}^t\right\}\nonumber  \\ 
        \geq& \; 
        \frac{r_j}{32}\sum_{t=1}^{T_j}\frac{1}{M_j}\sum_{l\neq k}\int\min\left\{d \Pr_{j,k,k}^t,d \Pr_{j,k,l}^t\right\}\nonumber \\
        \geq& \;  
        \frac{r_jT_j}{32}\cdot\frac{1}{M_j}\sum_{l\neq k}\int\min\left\{d \Pr_{j,k,k}^{T_j},  d\Pr_{j,k,l}^{T_j}\right\}\label{eq:lower_cons_1_1-dz} \\ 
        \geq& \;  
        \frac{r_jT_j}{32}\cdot\frac{1}{M_j}\sum_{l\neq k}\int_{A_j}\min\left\{ d \Pr_{j,k,k}^{T_j}, d \Pr_{j,k,l}^{T_j}\right\}\label{eq:lower_cons_1_2-dz} \\ 
        \geq& \;  
        \frac{r_jT_j}{32}\cdot\frac{1}{M_j}\sum_{l\neq k}\int_{A_j}\min\left\{ d \Pr_{j,k,k}^{T_{j-1}}, d \Pr_{j,k,l}^{T_{j-1}}\right\}, \label{eq:lower_cons_1-dz}
    \end{align}  
    where (\ref{eq:lower_cons_1_1-dz}) follows from data processing inequality of total variation and the equation $\int\min\left\{dP,dQ\right\}=1-TV(P,Q)$, (\ref{eq:lower_cons_1_2-dz}) restricts the integration to event $A_j$, and (\ref{eq:lower_cons_1-dz}) holds because the observations at time $T_j$ are the same as those at time $T_{j-1}$ under event $A_j$.
    
    For the term $\int_{A_j}\min\left\{d\Pr_{j,k,k}^{T_{j-1}},d\Pr_{j,k,l}^{T_{j-1}}\right\}$, it holds that 
    \begin{align}
        \int_{A_j}\min\left\{d\Pr_{j,k,k}^{T_{j-1}},d\Pr_{j,k,l}^{T_{j-1}}\right\} 
        =&\; \int_{A_j}\frac{d\Pr_{j,k,k}^{T_{j-1}}+d\Pr_{j,k,l}^{T_{j-1}}-\left|d\Pr_{j,k,k}^{T_{j-1}}-d\Pr_{j,k,l}^{T_{j-1}}\right|}{2}\nonumber\\ 
        =&\; 
        \frac{\Pr_{j,k,k}^{T_{j-1}}(A_j)+\Pr_{j,k,l}^{T_{j-1}}(A_j)}{2}-\frac{1}{2}\int_{A_j}\left|d\Pr_{j,k,k}^{T_{j-1}}-d\Pr_{j,k,l}^{T_{j-1}}\right|\nonumber\\
        \geq&\;  
        \(\Pr_{j,k,k}^{T_{j-1}}(A_j)-\frac{1}{2}TV \(\Pr_{j,k,k}^{T_{j-1}},\Pr_{j,k,l}^{T_{j-1}}\)\) 
        -TV \(\Pr_{j,k,k}^{T_{j-1}},\Pr_{j,k,l}^{T_{j-1}}\)\label{eq:lower_cons_2_1-dz}\\ 
        =&\; 
        \Pr_{j,k}(A_j)-\frac{3}{2}TV \( \Pr_{j,k,k}^{T_{j-1}},\Pr_{j,k,l}^{T_{j-1}} \), \label{eq:lower_cons_2-dz} 
    \end{align} 
    where (\ref{eq:lower_cons_2_1-dz}) uses the inequality $|\Pr(A)-\mathbb{Q}(A)|\leq TV(\Pr,\mathbb{Q})$, and (\ref{eq:lower_cons_2-dz}) holds because $I_{j,k}=I_{j,k,k}$ and $A_j$ can be determined by the observations up to $T_{j-1}$.
    
    We use an argument similar to (\ref{chain-dz}) to get 
    \begin{align*}
    	D_{KL}\(\Pr_{j,k,k}^{T_{j-1}}\|\Pr_{j,k,l}^{T_{j-1}}\)
    	\le
    	\frac{1}{2}\cdot\left(\frac{3r_j}{16}\right)^2\E_{\Pr_{j,k}}\tau_l
    	\le \frac{r_j^2}{32} \E_{\Pr_{j,k}}\tau_l,
    \end{align*}
    where $\tau_l$ denotes the number of pulls which is in $C_{j,l}$ before the batch of time $T_{j-1}$. Then from Lemma \ref{lem:pinsker-type}, we have
    \begin{align}
        \frac{1}{M_j}\sum_{l\neq k} TV \( \Pr_{j,k,k}^{T_{j-1}},\Pr_{j,k,l}^{T_{j-1}}\)
        \leq& \; 
        \frac{1}{M_j}\sum_{l\neq k}\sqrt{1-\exp\(-D_{KL}\(\Pr_{j,k,k}^{T_{j-1}}\|\Pr_{j,k,l}^{T_{j-1}}\)\)} \nonumber\\ 
        \leq& \; 
        \frac{1}{M_j}\sum_{l\neq  k}\sqrt{1-\exp\(-\frac{r_j^2}{32}\E_{\Pr_{j,k}}\tau_l\)} \nonumber\\ 
        \leq& \; 
        \frac{M_j-1}{M_j}\sqrt{1-\exp\(-\frac{r_j^2}{32(M_j-1)}\sum_{l\neq k}\E_{\Pr_{j,k}}\tau_l\)} \nonumber\\ 
        \leq& \; 
        \frac{M_j-1}{M_j}\sqrt{1-\exp\(-\frac{r_j^2 T_{j-1}}{32(M_j-1)}\)} \nonumber\\
        \leq& \; 
        \frac{M_j-1}{M_j}\sqrt{1-\exp\(-\frac{M_j}{32(M_j-1) B^2}\)} \label{eq:lower_cons_3_1-dz}\\
        \leq& \; 
        \frac{M_j-1}{M_j}\sqrt{\frac{M_j}{32(M_j-1) B^2}}\nonumber\\
        \leq& \; 
        \frac{1}{4B},\label{eq:lower_cons_3-dz} 
    \end{align}
    where (\ref{eq:lower_cons_3_1-dz}) uses (\ref{eq:rm-relation-dz}).
         
    Combining (\ref{eq:lower_cons_1-dz}), (\ref{eq:lower_cons_2-dz}) and  (\ref{eq:lower_cons_3-dz}) yields that
    \begin{align*} 
        \sup_{I\in\mathcal{I}_{j,k}}\E \[R_T(\pi) \] 
        \geq 
        \frac{1}{32}  r_jT_j\(\frac{\Pr_{j,k}(A_j)}{2}-\frac{3}{8B}\)\geq 
        \frac{1}{32B} T^{\frac{1-\varepsilon}{1 -  \varepsilon^B}}\(\frac{\Pr_{j,k}(A_j)}{2} -  \frac{3}{8B}\),
    \end{align*}  
    where $\varepsilon=\frac{1}{d_z+2}$. This inequality holds for any $k \leq M_j-1$. Averaging over $k$ yields
    \begin{align*}
        \begin{split}
            \sup_{I\in\cup_{k\leq M_j-1}\mathcal{I}_{j,k}}\E \[ R_T(\pi) \] 
            \geq& \; 
            \frac{1}{32B} T^{\frac{1-\varepsilon}{1-\varepsilon^B}} \( \frac{1}{2(M_j-1)}\sum_{k=1}^{M_j-1} \Pr_{j,k}(A_j)-\frac{3}{8B} \)\\ 
            \geq& \; 
            \frac{1}{32B} T^{\frac{1-\varepsilon}{1-\varepsilon^B}} \( \frac{7}{16B}-\frac{3}{8B} \)\\ 
            \geq& \;  
            \frac{1}{512B^2}T^{\frac{1-\varepsilon}{1-\varepsilon^B}}, 
        \end{split}
    \end{align*}
    where the second inequality holds from $p_j\geq\frac{7}{8B}$. Hence, the proof of Lemma \ref{adagrid_lemma-dz} is completed.
\end{proof}

Finally, combining the above two lemmas, we arrive at the lower bound in Theorem \ref{thm:lower-adaptive-dz}.
\end{proof}

\section{Additional Experimental results}
\label{app:blin}
\subsection{Repeated experiments of A-BLiN}
We present results of A-BLiN with some random seeds in Figure \ref{fig:regret-seeds}, where the figure legends and labels are the same as whose in Figure \ref{regret}. These results stably agree with the plot in the paper. The curve of zooming algorithm in Figure \ref{regret} and \ref{fig:regret-seeds} is the average of $10$ repeated experiments. The reason we did not present averaged regret curve of A-BLiN in Figure \ref{regret} is that we want to show the batch pattern of a single A-BLiN run in the figure. Averaging across different runs breaks the batch pattern. As an example, one stochastic run may end the first batch after 100 observations, while another may end the third batch after 110 observations.

\begin{figure*}[ht] 
    \centering
    \subfloat{ 
        \includegraphics[width=4.8cm]{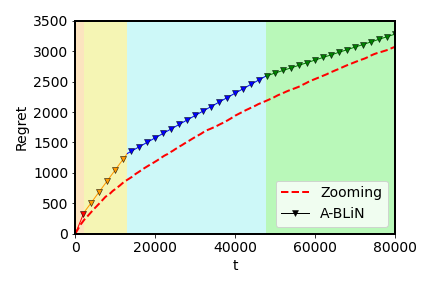}
    }
    \subfloat{
        \includegraphics[width=4.8cm]{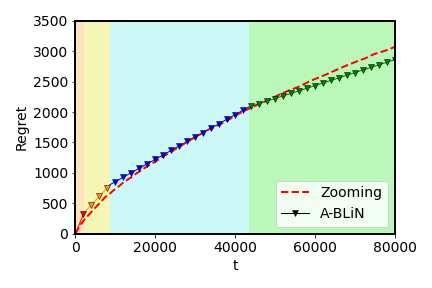}
    }
    \subfloat{
        \includegraphics[width=4.8cm]{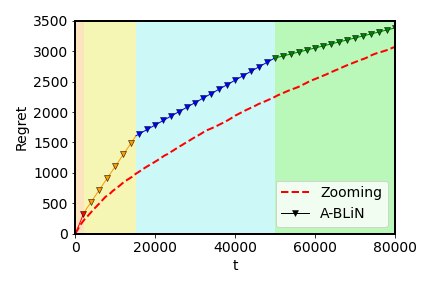}
    }\\
    \subfloat{
        \includegraphics[width=4.8cm]{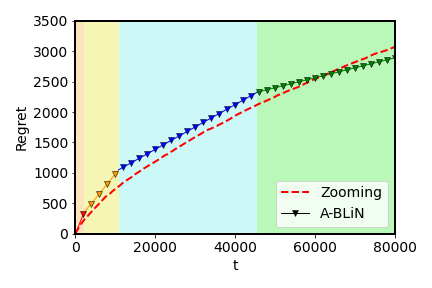}
    }
    \subfloat{
        \includegraphics[width=4.8cm]{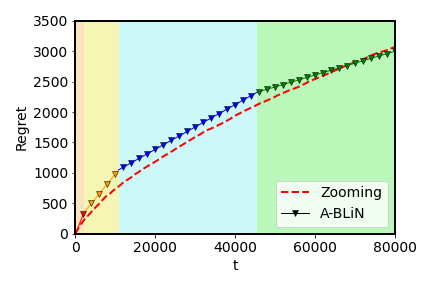}
    } 
    \subfloat{
        \includegraphics[width=4.8cm]{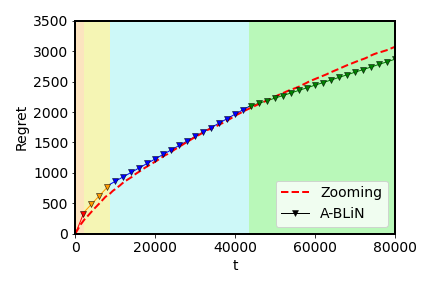}
    } 
    \caption{Results of A-BLiN with some random seeds. The figure legends and labels are the same as whose in Figure \ref{regret}.}
    \label{fig:regret-seeds}
\end{figure*} 

\subsection{Experimental results of D-BLiN}\label{app:exp-dblin}
We run D-BLiN to solve the same problem in Section \ref{exp}. The partition and elimination process of this experiment is presented in Figure \ref{app:partition}, which shows that the optimal arm $x^*$ is not eliminated during the game, and only $6$ rounds of communications are needed for time horizon $T=80000$. Moreover, we present the resulting partition and the accumulated regret in Figure \ref{app:parti-regret}. 

\begin{figure}[htbp]
    \centering
    \subfloat{ 
    	\includegraphics[width=3.5cm]{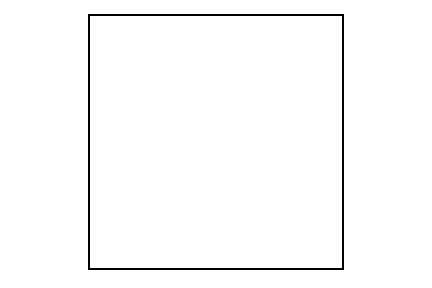}
    }\hspace{-15mm}
    \subfloat{ 
    	\includegraphics[width=3.5cm]{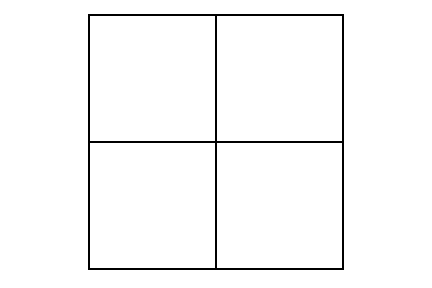}
    }\hspace{-15mm}
	\subfloat{
		\includegraphics[width=3.5cm]{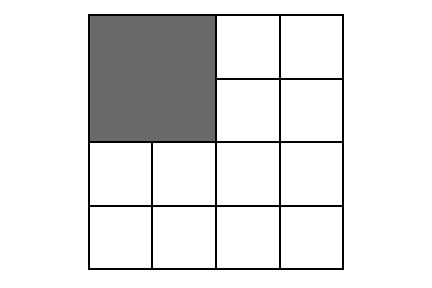}
	}\hspace{-15mm}
	\subfloat{
		\includegraphics[width=3.5cm]{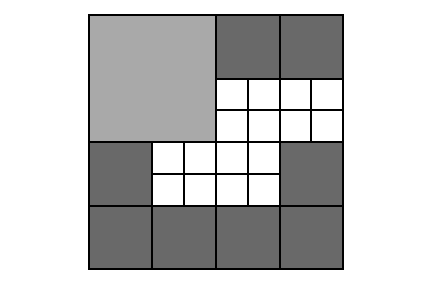}
	}\hspace{-15mm}
	\subfloat{
		\includegraphics[width=3.5cm]{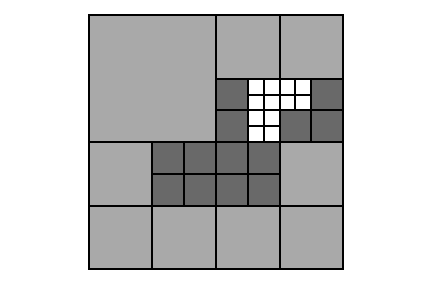}
	}\hspace{-15mm}
	\subfloat{
		\includegraphics[width=3.5cm]{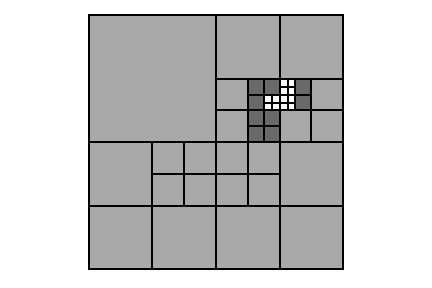}
	}
	\caption{Partition and elimination process of D-BLiN. The $i$-th subfigure shows the pattern before the $i$-th batch. Dark gray cubes are those eliminated in the most recent batch, while the light gray ones are those eliminated in earlier batches. For the total time horizon $T= 80000$, D-BLiN needs $6$ rounds of communications. 
	} 
	\label{app:partition}
\end{figure}
\begin{figure}[htbp]
	\centering
	\subfloat[Partition]{
		\includegraphics[height=5cm]{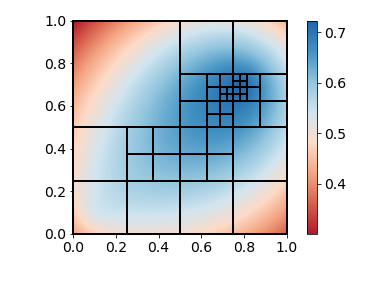}
		\label{app_parti}
	}
	\subfloat[Regret]{
		\includegraphics[height=5cm]{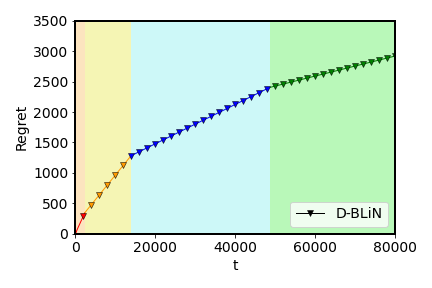}
		\label{app_regret}
	}
	\caption{Resulting partition and regret of D-BLiN. In Figure \ref{app_parti}, we show the resulting partition of D-BLiN. The background color denotes the true value of expected reward $\mu$, and blue means high values. The figure shows that the partition is finer for larger values of $\mu$. In Figure \ref{app_regret}, we show accumulated regret of D-BLiN. In the figure, different background colors represent different batches. For the total time horizon $T=80000$, D-BLiN needs $6$ rounds of communications (the first two batches are too small and are combined with the third batch in the plot).} 
	\label{app:parti-regret}
\end{figure}

\section{Auxiliary Lemmas}
\label{app:lemmas}
\begin{lemma}[Bretagnolle-Huber Inequality\cite{bretagnolle1978estimation}]
    \label{lem:pinsker-type}
    Let $P$ and $Q$ be any probability measures on the same probability space. It holds that 
    \begin{align*}
        TV(P, Q) 
        \le   
        \sqrt{ 1 - \exp \( - D_{KL} (P \| Q) \) } 
        \le   
        1 - \frac{1}{2} \exp \left(- D_{KL}( P \| Q) \right).
    \end{align*} 
\end{lemma} 

\begin{lemma}[\cite{gao2019batched}]
    \label{lem:test-tree}
    Let $Q_1, \cdots , Q_n$ be probability measures over a common probability space $ (\Omega, \mathcal{F}) $, and $\Psi : \Omega \rightarrow [n]$ be any measurable function (i.e., test). Then for any tree $T = ([n], E)$ with vertex set $[n]$ and edge set $E$, we have
    \begin{enumerate}
        \item $\frac{1}{n} \sum_{ i = 1 }^n Q_i ( \Psi \neq i ) 
        \ge 
        \frac{1}{n} \sum_{ (i,j) \in E } \int \min \{ d Q_i, d Q_j \} ; $  
        \item $\frac{1}{n} \sum_{ i = 1 }^n Q_i ( \Psi \neq i ) 
        \ge 
        \frac{1}{2n} \sum_{ (i,j) \in E } \exp \left( - D_{KL} (Q_i \| Q_j ) \right). $
    \end{enumerate} 
\end{lemma}

\end{document}